\documentclass[twoside,11pt]{article}

\usepackage{jair, theapa, rawfonts}
\usepackage{aliascnt}
\usepackage{algorithm}
\usepackage{algorithmic}
\usepackage{graphics}
\usepackage{epsfig}
\usepackage{times}
\usepackage{tabls}
\usepackage{tabularx}
\usepackage{subfigure}
\usepackage{latexsym}
\usepackage{fix-cm}
\usepackage{plain}
\usepackage{dsfont}
\usepackage{graphicx}
\usepackage[fleqn]{amsmath}
\usepackage{amssymb}
\usepackage{xspace}
\usepackage{tabls}
\usepackage{multicol}
\usepackage{tabularx}
\usepackage{url}
\usepackage{multicol}
\usepackage{multirow}
\usepackage{color}
\usepackage{amsmath}
\usepackage{amsthm}
\usepackage{amssymb}

\ShortHeadings{Sematic sort: A supervised approach to personalized semantic relatedness}
{El-Yaniv, \& Yanay}
\firstpageno{1}

\begin{document}

\newtheorem{theorem}{Theorem}
\newtheorem{lemma}[theorem]{Lemma}
\newtheorem{proposition}[theorem]{Proposition}
\newtheorem{corollary}[theorem]{Corollary}
\newtheorem{observation}[theorem]{Observation}
\newtheorem{definition}{Definition}
\newcommand{\BlackBox}{\rule{1.5ex}{1.5ex}}  % end of proof

\newcommand{\nsd}{\mathsf{NSD}}
\newcommand{\wnsd}{\mathsf{WNSD}}
\newcommand{\lambdaUP}{\lambda_{up}}
\newcommand{\lambdaDN}{\lambda_{dn}}

\def\algname{\operatorname{SemanticSort}}

\newcommand\red[1]{\color{red}#1}
\newcommand{\eps}{\epsilon}
\newcommand{\alg}{\textsc{alg}}
\newcommand{\cl}{{\ell}}
\newcommand{\cA}{{\cal A}}
\newcommand{\cC}{{\cal C}}
\newcommand{\cD}{{\cal D}}
\newcommand{\cP}{{\cal P}}
\newcommand{\cF}{{\cal F}}
\newcommand{\cH}{{\cal H}}
\newcommand{\cX}{{\cal X}}
\newcommand{\cY}{{\cal Y}}
\newcommand{\cG}{{\cal G}}
\newcommand{\cS}{\Sigma}

\newcommand{\Bf}{\mathbf{f}}
\newcommand{\Bg}{\mathbf{g}}
\newcommand{\BX}{\mathbf{X}}
\newcommand{\BY}{\mathbf{Y}}
\newcommand{\BD}{\mathbf{D}}
\newcommand{\I}{\mathbb{I}}
\newcommand{\C}{\mathcal{C}}
\newcommand{\N}{\mathcal{N}}
\newcommand{\cO}{\cal O}
\newcommand{\E}{\mathbf{E}}
\newcommand{\err}{\mathop{\rm er}}
\newcommand{\hf}{\hat{f}}
\newcommand{\hg}{\hat{g}}
\newcommand{\LA}{L}
\newcommand{\GLi}{G_L^{(i)}}
\newcommand{\GUi}{G_U^{(i)}}
\newcommand{\fail}{\texttt{fail}}
\newcommand{\reals}{\mathbb{R}}
\newcommand{\eqdef}{\triangleq}

\newcommand{\proc}{\textbf{Procedure\xspace}}
\newcommand{\init}{\textbf{Initialize\xspace}}
\newcommand{\myloop}{\textbf{Loop\xspace}}
\newcommand{\until}{\textbf{Until\xspace}}
\newcommand{\while}{\textbf{While\xspace}}
\newcommand{\procname}[1]{\textsc{#1}}
\newcommand{\procargs}[1]{$(#1)$}
\newcommand{\procarg}[1]{\textit{#1}}
\newcommand{\procfor}{\textbf{For\xspace}}
\newcommand{\procforeach}{\textbf{Foreach\xspace}}
\newcommand{\procif}{\textbf{If\xspace}}
\newcommand{\procor}{\textbf{Or\xspace}}
\newcommand{\procelse}{\textbf{Else\xspace}}
\newcommand{\procreturn}{\textbf{Return\xspace}}
\newcommand{\procnorm}{\textbf{Normalize weights\xspace}}
\newcommand{\TABLINE}[1]{\hspace*{#1}\=\hspace{#1}\=\hspace{#1}\=\hspace{#1}\=\hspace{#1}\=\hspace{#1}\=\hspace{#1}\=\hspace{#1}\=\hspace{#1}\=\hspace{#1}\=\hspace{#1}\=\kill}
\newcommand{\MyNode}[1]{\Tr{\psframebox[fillcolor=white!50,fillstyle=solid]{#1}}}

\title{Semantic Sort: A Supervised Approach to Personalized Semantic Relatedness}

\author{\name Ran El-Yaniv \email rani@cs.technion.ac.il \\
       \name David Yanay \email dudu.yanay@gmail.com \\
       \addr Department of Computer Science,\\
       Technion - Israel Institute of Technology}

\maketitle

\begin{abstract}
We propose and study a novel supervised approach
to learning statistical semantic relatedness models from subjectively annotated training examples.
The proposed semantic model consists of parameterized
co-occurrence statistics associated with
textual units of a large background knowledge corpus.
We present an efficient algorithm for learning such semantic
models from a training sample of relatedness preferences.
Our method is corpus independent and can essentially rely on any sufficiently large (unstructured)
collection of coherent texts.
Moreover, the approach facilitates the fitting of semantic models for specific users or groups of users.
We present the results of extensive range of experiments from small to large scale,
indicating that the proposed method is effective and competitive with the state-of-the-art.
\end{abstract}

\section{Introduction}
\label{Section:Introduction}
In recent years the problem of automatically determining semantic relatedness has been steadily gaining attention among statistical NLP and AI researchers.
This surge in semantic relatedness research has been reinforced by the emergence of applications that can greatly benefit
from semantic relatedness capabilities. Among these applications we mention targeted advertising
\shortcite{Broder2007,Ribeiro-Neto2005}, information retrieval and web search
\shortcite{Egozi2008,Varelas2005,Guha2003,Richardson95,Srihari2000},
automatic tagging and linking \shortcite{KorenLMS11,Schilder2001,Miller1993,Green1999}, and text categorization
 \shortcite{Gabrilovich2006,Sebastiani2002,Gabrilovich2005}.

To motivate the need for semantic analysis capabilities, consider, for example, the difficult task of categorizing short text units (e.g., ads, tweets, search queries, reviews, etc.) using supervised learning.
Each specific unit contains very few words, and therefore we can find many units expressing the same idea (and belonging to the same category),
which only share function words (e.g., stop-words), but not content words.
The pedestrian approach, based on bag-of-words representation, might not be effective in this task
because short text units to be categorized often do not share many words with the training examples in their category.
It is now clear that it is necessary to represent such texts using semantic features
\shortcite<see, e.g.,>{Sriram2010,MihalceaCS06,Sun2011,Turney02,Yuhua2006}.
In general, many other applications require some form of deep semantic processing
and cannot rely only on shallow syntactical considerations.

%The NLP literature emphasizes at least three basic semantic tasks.
%The first task is \emph{sense / semantic disambiguation}, where the goal is to identify the intended meanings
%(senses) of the occurrences of ambiguous target terms. Thus, this task does not consider the relations between two terms, but rather, the ambiguity of a single term.
%Disambiguation of a term is done with respect to a specific \emph{context} containing other terms. The co-occurrence
%patterns of those terms appearing together with the target term enables disambiguation,
%see, e.g., \cite{Lesk1986,Cowie1992,Yarowsky1995,Agirre1996,Schutze1998,LeacockMC1998,Ide1998,Navigli2010}.

In \emph{semantic relatedness} the goal is to quantify the intensity of how much two terms are related to each other.
The relatedness task considers all relation types between two target terms. These relations
can be among the known formal linguistics ones, which have a name (such as synonyms, antonyms, hypernyms, meronyms, related nouns, etc.),
but in general, such relations can be informal in the sense that they do not have a given name and they express some
(perhaps complex) relation between the two terms that has not been studied.
For example, consider the following three term pairs
\begin{eqnarray*}
\texttt{Michael Jordan} & : & \texttt{Basketball} \\
\texttt{Madonna} & : & \texttt{Pop} \\
\texttt{Marilyn Monroe} & : & \texttt{Movie}.
\end{eqnarray*}
All three pairs, \texttt{X : Y}, are strongly related via a common relation.
What would be your assessment of an underlying relation for these three pairs?\footnote{The common relation
we had in mind is ``X is an all times Y star''.}

To summarize, the semantic relatedness task involves all possible relations whose number is in principle unbounded.
We note that the
The NLP literature also considers the task of \emph{semantic similarity} in which the
(rather limited) goal  is to quantify the \emph{synonymy} relation between two terms.
Indeed, as argued by \citeA{BudanitskyH06}, semantic relatedness
is considered more general than semantic similarity. In this sense the general semantic relatedness task is
more difficult.

In this work we consider the semantic relatedness task and we aim at qualifying the relatedness of two given terms
where the underlying relation can be formal or informal.
However, we do not aim at identifying or characterizing the underlying relation.\footnote{Such a relation \emph{characterization} task
is a very interesting problem in and off itself, but is far beyond the scope of our work.}
Note also that
in the standard semantic relatedness setting we consider here (see definitions in Section~\ref{Section:ProblemSetup}),
the terms to be evaluated for relatedness are provided without a context, unlike standard disambiguation settings
\cite<see, e.g.,>{Lesk1986,Cowie1992,Yarowsky1995,Agirre1996,Schutze1998,LeacockMC1998,Ide1998,Navigli2010}.
Thus, as most existing works on semantic relatedness, focusing on our or equivalent setup,
we do not aim at directly solving the disambiguation problem along the way.\footnote{Nevertheless, we note that we believe that it is possible
to extend our techniques to disambiguate a term within a context.}

Semantic relatedness is an elusive concept.
While a rigorous mathematical definition of semantic relatedness is currently beyond grasp, the concept
is intuitively clear.
%and moreover, is heaviliy
%Indeed, the popular conception of semantic relatedness is reflected in its nebulous Wikipedia entry stating that:
%\emph{In essence, semantic similarity, semantic distance, and semantic relatedness all mean, ``How much does term A have to do with term B''? }
%Clearly, semantic relatedness ``has to do'' with the understanding of meaning -- a grand challenge in AI research.
%Can a computer program quantify the extent to which two terms share the same meaning?
Moreover, humans exhibit remarkable
capabilities in processing and understanding textual information, which is partly related to
their ability to assess semantic relatedness of terms.
Even without precise understanding of this intelligent ability, it is still intuitively clear that
deep semantic processing
of terms and text fragments should heavily rely on background knowledge and experience.
%Can we design a computer program that also use background knowledge and mimic
%such capabilities?

The statistical NLP and AI communities have adopted a pragmatic modus operandi to these questions: even if we don't know how to
define semantic relatedness, we can still create computer programs that emulate it.
Indeed, a number of useful heuristic approaches to semantic relatedness have been
proposed,  and this line of work has proven to be rewarding
\cite<see, e.g.,>{Resnik95,DekangLin657297,Bloehdorn2007,GabrilovichM07,Cowie1992}.
In particular, it has been shown that useful semantic relatedness scores can be systematically extracted
from large lexical databases or electronic repositories of common-sense and domain-specific background knowledge.

With the exception of a few papers,
most of the algorithms proposed for semantic relatedness
valuation have been following \emph{unsupervised} learning or \emph{knowledge engineering} procedures.
Such semantic relatedness valuation functions have been
generated, for the most part,  using some hand-crafted formulas applied to semantic information
 extracted from a (structured) background knowledge corpus.
The proposed methods have employed a number of interesting techniques,
 some of which are discussed in Section~\ref{Section:RelatedWork}.

One motivation for the present work is the realization that semantic relatedness assessments are \emph{relative} and \emph{subjective}
rather than \emph{absolute} and \emph{objective}. While we can expect some kind of consensus among people on the (relative) relatedness valuations of
basic terms, the relatedness assessments of most terms depend on many subjective and personal factors such as literacy, intelligence, context, time and location.
For example, the name \texttt{Michael Jordan} is generally strongly related to \texttt{Basketball},
but some people in the machine learning community may consider it more related to \texttt{Machine Learning}.
 As another example, consider WordSim353 \shortcite{FinkelsteinGMRSWR01i}, the standard benchmark dataset for evaluating and comparing semantic relatedness measures (see Section~\ref{Section:StandardBenchmarkDatasets}).
This benchmark contains some controversial relative preferences between word pairs such as
\begin{eqnarray*}
\texttt{Arafat-Peace}& vs. & \texttt{Arafat-Terror} \\
\texttt{Jerusalem-Israel} & vs. & \texttt{Jerusalem-Palestinian}.
\end{eqnarray*}
Can you tell which pair is more related in each instance?
Obviously, the answer must be personal/subjective.
As a final example for the subjective nature of semantic relatedness, let's consider the Miller and Charles's dataset \cite{MillerCharlesDataset}, which is a distinct subset of the Rubenstein and Goodenough's dataset \cite{Rubenstein1965}.
 Both datasets were annotated using the same score scale by (probably different) human raters.
 This double rating resulted in different semantic scores and more importantly, in different
 pair rankings.\footnote{The Spearman correlation between the rankings  of these datasets is 0.947.}
 It is evident that each dataset expresses the subjective semantics of its human raters.

This sensitivity of semantic relatedness to subjective factors should make it very hard, if not impossible,
 to satisfy all semantic relatedness needs using an unsupervised or a hand-crafted method.
 Moreover, the fitting to a particular test benchmark in an unsupervised manner
is not necessarily entirely meaningful in certain scenarios.
Indeed, some published semantic relatedness measures outperform others in
certain benchmarks tests and underperform in others.
For example, \citeA{PonzettoS07} mentioned that the WordNet-based measures
perform better than the Wikipedia-based measures on the Rubenstein and Goodenough benchmark, but the WordNet methods are inferior over
WordSim353.

In this work we propose a novel \emph{supervised} approach to learning semantic relatedness from examples.
Following \shortciteA{1620758} we model semantic relatedness learning as a binary classification problem
where each instance encodes the relative relatedness of two term pairs.
Given a labeled training set our goal is to learn a semantic relatedness function capable of determining the labels of unobserved instances.
We present an empirical risk minimization (ERM) algorithm that learns by inducing
a weighted measure of terms co-occurrence defined over a background knowledge corpus of free-text documents.
The labeled examples are used to fit this model to the training data.
The resulting algorithm is relatively simple, has only few hyper-parameters, and is
corpus independent. Our experiments show that the algorithm achieves notable generalization performance.
This is observed over a wide range of experiments on a number of benchmarks.
We examine and demonstrate the effectiveness of our algorithm
using two radically different background knowledge corpora: an old version of Wikipedia and
the books in the Project Gutenberg.

\section{Related Work}
\label{Section:RelatedWork}
The literature survey in this section attempts to encompass techniques and algorithms for assessing semantic relatedness.
As the class of such techniques is quite large, the discussion here is limited to ideas and works in close vicinity
of the present work.
Semantic relatedness techniques typically rely on some kind of world or expert knowledge, which we term here
\emph{background knowledge (BK)}. The BK is a key element in many methods and
we categorize semantic relatedness techniques into three main types according to type and structure of
their BK.
\emph{Lexical} methods rely on lexical databases such as WordNet \cite{WordNet} (George A. Miller began the WordNet project in the mid-1980s) or Roget's Thesaurus \cite{Roget1852}.
\emph{Wiki} methods rely on structured BK corpora like Wikipedia or the Open Directory Project (DMOZ).
The structure in Wiki BKs can be manifested in various ways, and the most important ones are semantic coherency of documents and titles,
meaningful interlinks (often accompanied with meaningful anchor texts), and hierarchical categorization.
Finally, semantic relatedness techniques that rely on unstructured text collections are referred to as \emph{structure-free} methods.
Before delving into these three BK types we divert the discussion in the next subsection and
elaborate on standard benchmark datasets for evaluating semantic relatedness techniques.

\subsection{Standard Benchmark Datasets}
\label{Section:StandardBenchmarkDatasets}
A key contributing element that greatly influenced semantic relatedness research is the presence of benchmark test collections.
While the currently available datasets are quite small, they are considered ``representative'' and meaningful because they were annotated by human raters.
Each of these datasets consists of a list of word pairs, along with their numerical
relatedness score.
In the semantic relatedness literature it is common to evaluate relatedness ranking, $Y = \{y_i\}_1^n$, with the corresponding ground truth
(conveyed by such datasets), $Z = \{z_i\}_1^n$,  using the Spearman correlation, defined as,
$$
\rho(Y,Z)=1 - \frac{6 \cdot \sum_1^n (y_i - z_i)^2}{n \cdot (n^2 - 1)}.
$$
Rubenstein and Goodenough (R\&G) \cite{Rubenstein1965} were perhaps the first to assemble an annotated  semantic dataset.
Their dataset consists of 65 word pairs associated with their similarity scores,
where mark $4$ is assigned to the most similar pairs (often synonyms), and mark $0$ to the least similar ones.
Miller and Charles (M\&C) \cite{MillerCharlesDataset} selected a particular subset from the R\&G set consisting of 30 word pairs,
which were than ranked using the same 0--4 score scale.

WordSim353 is the most recent semantic benchmark dataset \cite{FinkelsteinGMRSWR01i}.\footnote{WordSim353 is available at \url{http://www.cs.technion.ac.il/~gabr/resources/data/wordsim353}.}
This dataset, while still small, is substantially larger and consists of a list of 353 word pairs along with their human evaluated semantic relatedness scores, from $0$ (the least related) to $10$ (the most related).
While the R\&G and M\&C datasets are used for evaluating semantic similarity measures (i.e., synonym relations), WordSim353 involves a variety of semantic relations and in the past years
has been providing a focal point to practical semantic relatedness research. In the discussion below we will mention WordSim353 Spearman correlation scores in cases where they were reported.

\subsection{Lexical Methods}
Many of the lexical methods rely on the WordNet database \cite{WordNet} as their BK corpus.
WordNet is a lexical database for the English language that was created and is being maintained at the Cognitive Science Laboratory in Princeton University.
WordNet organizes English words in groups called \emph{synsets}, which are sets of synonyms.
The lexical relations between synsets are categorized into types such as hypernyms, meronyms, related nouns, ``similar to'',
etc.\footnote{
Y is a hypernym of X if every X is a (kind of) Y. See definitions of the rest of these linguistic relations in \url{http://en.wikipedia.org/wiki/WordNet}.}
 In addition to these semantic relations WordNet also provides a polysemy count (the number of synsets that contain the term) for disambiguation.
 WordNet is intended to be used both manually or automatically to serve applications.

Another lexical database is Roget's Thesaurus \cite{Roget1852}. Although it might be implied from its name, this database is not a dictionary of synonyms,
and as stated by Kirkpatrick \cite{Kirkpatrick1998}: ``\emph{it is hardly possible to find two words having in all respect the same meaning, and being therefore
interchangeable}.''
Similarly to WordNet, Rodget's Thesaurus contains groups of terms, called \emph{semicolon groups}, which are linked. However, those links are not
lexically annotated as in WordNet.

Lexical semantic relatedness methods typically view the lexical database as a graph whose nodes are terms and edges are lexical relations.
Semantic relatedness scores are extracted using certain statistics defined on this graph.

\citeA{Resnik95,Resnik99} generated semantic relatedness scores based on a combination of IS-A (hyponym) relations in WordNet and a structure-free corpus.
Each synset in WordNet, $c$, is assigned a probability, $prob(c)$, according to the frequency of its descendants (including itself) in a corpus.
The information content (ic) of two synsets $c_1$ and $c_2$ is then defined as
$ic(c_1,c_2) = \max_{c \in \Psi(c_1,c_2)} {\{ - \log(prob(c))\}}$,
where $\Psi(c_1,c_2)$ is the set of synsets that are connected by an IS-A directed path to both $c_1$ and $c_2$; that is,
 $\Psi(c_1,c_2)$ is the set of all the ancestors of both $c_1$ and $c_2$.
The semantic relatedness (sr) of two terms, $t_1$ and $t_2$, is defined as
$$
sr(t_1,t_2) = \max_{c_1 \in s(t_1), \ c_2 \in s(t_2)} {\{ ic(c_1,c_2)\}},
$$
where $s(t)$ is the set of synsets in WordNet that contain $t$.

Another attempt to combine WordNet IS-A relations with a structure-free corpus was made by
% Jiang and Conrath
\citeA{journals/corr/cmp-lg-9709008}.
They weighted a link (lexical relation) between a child node, $c$, and a parent node, $p$, according to the differences in their information content (as proposed by Resnik),
the depth of $p$ in the hierarchy, the degree of $p$, and the average degree in the whole hierarchy.
The semantic relatedness of two terms, $t_1$ and $t_2$, is valuated by summing up the weights along the shortest path between a synset that contains $t_1$, and a
synset containing $t_2$. Utilizing their measure, Jiang and Conrath managed to improve upon the Resnik measure.

%Li et al.
\shortciteA{TKDE2003} defined and calculated the semantic similarity between two words, $w_1$ and $w_2$,  as a function of: (i) the shortest path between $w_1$ and $w_2$;
(ii) the depth of the first concept in the IS-A hierarchy that subsumes both $w_1$ and $w_2$; and (iii) the semantic density of $w_1$ and $w_2$, which is based on their information content.
Li et al. assumed that these three information sources are independent and used several nonlinear functions to combine them.

%Banerjee and Pedersen
\citeA{BanerjeeP03} extended the glosses (definitions in WordNet) overlap measure defined by
%Lesk
\citeA{Lesk1986}. Given two synsets in WordNet, they enriched their
glosses with the glosses of their related synsets, according to WordNet link structure, and calculated semantic relatedness
as a function of the overlap between these
``enriched glosses.''
Banerjee and Pedersen also weighted the terms in the overlap according to the number of words in those terms.
\citeA{Pedersen06}
%Patwardhan and Pedersen
combined co-occurrences in raw text with WordNet definitions to build gloss vectors.

%Jarmasz and Szpakowicz
\citeA{Jarmasz03} calculated the semantic relatedness between two terms as the number of edges in all the pathes between the two terms in a Roget's Thesaurus graph,
achieving 0.54 correlation with WordSim353 dataset \cite{JarmaszThesis}.

%Hughes and Ramage
\citeA{HughesR07} calculated the Personalized PageRank vector \shortcite{PersonalizedPageRank} for each node (term) under some representation of WordNet as a graph.
They considered three node types: (i) synsets; (ii) \emph{TokenPOS}, for a word coupled with its part-of-speech tag;
and (iii) \emph{Token}, for a word without its part-of-speech tag. In addition to WordNet's links, each synset is connected to all the tokens in it or in its gloss. Moreover, they proposed three models to compute the stationary distribution: (i) \emph{MarkovLink},
which contains WordNet's links and links from tokens to synsets that contain them;
(ii) \emph{MarkovGloss}, containing only links between tokens and synsets that contain them in their gloss;
and (iii) \emph{MarkovJoined}, containing all the edges in both MarkovLink and MarkovGloss.
In order to estimate the similarity between two PageRank vectors, they used the cosine similarity measure, as well as a
newly proposed \emph{zero-KL Divergence} measure, based on Kullback-Leibler (KL) divergence measure of information theory.
Hughes and Ramage obtained their best result of 0.552 Spearman correlation with WordSim353,
when using the MarkovLink model and the zero-KL Divergence.

%Tsatsaronis et al.
\shortciteA{Tsatsaronis2009,TsatsaronisVV10} proposed the \emph{Omiotis} measure. They weighted the relations between synsets in WordNet according to their frequency. Given a WordNet path, $p$, between two synsets, $s_1$ and $s_2$, they defined its \emph{semantic compactness measure (SCM)} as the product of edge weights in $p$. In addition, they defined the \emph{semantic path elaboration (SPE)} of $p$ as:
$SPE\left( p \right) = \prod\nolimits_{i = 1}^l {\frac{{2{d_i} \cdot {d_{i + 1}}}}{{d_i} + {d_{i + 1}}} \cdot \frac{1}{d_{\max }}}$, where $d_i$ is the depth in WordNet of the synset $s_i$ in $p$,
 and $d_{\max }$ is the maximum depth. The compactness of $p$ is thus the product of the harmonic mean of
 depths of consecutive edges, normalized by the maximum depth. The semantic relatedness between $s_1$ and $s_2$
according to $p$ is $SCM\left( p \right) \cdot SPE\left( p \right)$. Finally, they defined the semantic relatedness between $s_1$ and $s_2$ as,
$\mathop {\max }\limits_{p \in P} \left\{ {SCM\left( p \right) \cdot SPE\left( p \right)} \right\}$, where $P$ is the set of all paths between $s_1$ and $s_2$.
Omiotis achieved 0.61 spearman correlation with WordSim353.

%Morris and Hirst
\citeA{Morris1991} introduced the concept of lexical chains between words as an element to represent and find the text structure.
They argued that coherent text is assembled from textual units (sentences and phrases) that convey similar meaning.
% and together function as a whole.
They termed these sequences of textual units as \emph{lexical chains}.
Using these chains they defined text cohesion and determined its meaning.
%Hirst and St-Onge
\citeA{HirstStOnge1998} constructed these chains from the links between WordNet synsets.

The reader is referred to \cite{Budanitsky01,BudanitskyH06} for a study of various other lexical methods. Refer also to
\cite{PedersenPM04} for a freely available software that implements six semantic measures: three information content based measures \cite{Resnik95,DekangLin657297,journals/corr/cmp-lg-9709008}, two path length based measures \cite{Leacock1998,WuP94}, and a baseline measure that is the inverse of the length of the shortest path between two concepts.

\subsection{Wiki Methods}
%Strube and Ponzetto
\citeA{WikiRelate2006,PonzettoS07} are perhaps the first to consider Wikipedia as the source for semantic relatedness information.
The relatedness between two terms $t_1$ and $t_2$ is computed
by identifying representative Wiki articles $d_1$ and $d_2$ containing those terms in their titles,
respectively.\footnote{In cases of multiple representative articles, several heuristics were proposed to resolve ambiguity.}
The semantic relatedness is then derived in several ways using several distance measures between $d_1$ and $d_2$, such as
normalized path-length in the category hierarchy \cite{Leacock1998},
information content \cite{Resnik95},
text overlap (number of common terms, proposed by \citeA{Lesk1986} and \citeA{BanerjeeP03}), etc.
The best result of this method (called Wikirelate!) achieved a 0.49 Spearman correlation with WordSim353.

%Gabrilovich and Markovitch
\citeA{GabrilovichM07} introduced the celebrated Explicit Semantic Analysis (\emph{ESA}) method,
where each term $t$ has distributional representation $v(t)$ over all Wikipedia articles.
The components of the vector $v(t)$ are TF-IDF scores \cite{Salton1988513} of the term $t$ in all articles.
The semantic relatedness value of two terms is defined as the cosine of their vectors.
Various enhancements and extensions to this basic ESA method were discussed in \cite{GabrilovichM09};
for example, a filter based on link analysis was introduced to obtain more meaningful distributional term representations.
ESA achieved a Spearman correlation of 0.75 with WordSim353, and is currently widely recognized as a top performing
semantic relatedness method. Moreover, ESA is frequently used as a subroutine in many applications
\cite<see, e.g.,>{Yeh:2009:WRW:1708124.1708133,RadinskyAGM11,Haralambous2011}.

%Milne and Witten
\citeA{MilneWitten2008} proposed the \textit{Wikipedia Link-based Measure (WLM)},
which utilizes the interlinks between Wikipedia's articles.
They proposed two methods to calculate the relatedness between two articles.
The first calculates a weighted vector of the links of each article and returns the cosine of these vectors.
The link weighting function is inspired by the TF-IDF measure.
The second method utilizes the Normalized Google Distance of \citeA{GoogleSimilarityDistance}
(discussed in Section~\ref{Section:AdaptiveCoOccurrenceMeasure}), applied to interlinks counts.
Given two terms, WLM selects two representing articles to these terms and returns the average of the above
methods.\footnote{Milne and Witten also proposed several ways to choose representative articles for a given pair of terms.}
WLM achieved a Spearman correlation of 0.69 with WordSim353.

%Yeh et al.
\shortciteA{Yeh:2009:WRW:1708124.1708133} proposed a method called \textit{WikiWalk}
that utilizes Wikipedia as a graph whose nodes are articles and the interlinks are the edges.
Given a text fragment, WikiWalk maps it to a distribution over the nodes and calculates its Personalized PageRank in the graph according to this distribution.
Yeh et al. proposed two methods to map the given text to a distribution over nodes: dictionary based, and \textit{ESA} based.
The semantic relatedness of two terms is defined as the cosine similarity between their Personalized PageRank vectors.
WikiWalk achieved a Spearman correlation of 0.634 with WordSim353.

%Radinsky et al.
\shortciteA{RadinskyAGM11} proposed the Temporal Semantic Analysis (\emph{TSA}), which expands the ESA method mentioned above by
adding a temporal dimension to the representation of a concept.
As in ESA, TSA represents terms as a weighted concept vector generated from a corpus.
However, for each concept, TSA extracts in addition a temporal representation using another corpus whose documents are divided into epochs (e.g., days, weeks, months, etc.).
With this extra corpus TSA calculates for each concept its ``temporal dynamics,'' which is its frequency in each epoch.
Given two terms, TSA computes their semantic relatedness by measuring the similarity between the temporal representation of their ESA concepts.
TSA obtained 0.82 correlation score with WordSim353, which is the best known result for WordSim353 using an \emph{unsupervised}
learning method.

\subsection{Structure-Free Methods}
Motivated by Kolmogorov complexity arguments, \citeA{GoogleSimilarityDistance} introduced a
novel structure-free semantic relatedness method, which is essentially a normalized co-occurrence measure.
This method, called the ``Google similarity distance,'' originally used the entire web as the unstructured corpus and
relied on a search engine to provide rough assessments of co-occurrence counts.
This method is extensively used in our work (without reliance on the entire web and search engines) and is
described in Section~\ref{Section:AdaptiveCoOccurrenceMeasure}.

Using term-document occurrence count matrix, \shortciteA{Deerwester90} used
Singular Value Decomposition (SVD) to compare the meaning of terms. Applying this measure,
\shortciteA{FinkelsteinGMRSWR01i}
achieved 0.56 correlation with WordSim353.\footnote{The implementation they used is available online at \url{http://lsa.colorado.edu}}

\citeA{DekangLin657297} proposed information-based methods to define and quantify term similarity.
\shortciteA{Dagan+Lee+Pereira:99a} and \citeA{Terra2003}
experimented with various statistical and information co-occurrence measures, such as mutual information, likelihood ratio, $L_1$ norm, and the KL- and Jensen-Shannon divergences, for estimating semantic relatedness from structure-free corpora.

%Reisinger and Mooney
\citeA{ReisingerMooney10} generated for a term $t$ a number of feature vectors, one for each context in which $t$ appears.
The feature vector of a certain context contains weights for all terms appearing with t in this
context and weights are calculated based on TF-IDF and $\chi^2$ scores.
These feature vectors were then clustered and cluster centroids were taken to represent the meaning of $t$.
Using this representation they considered various methods to calculate semantic relatedness of two terms according to similarity of their respective centroids.
This method obtained a correlation of 0.77 with WordSim353 by using combined centroids from clusterings of different resolutions.
This impressive performance is among the best known.

\subsection{Supervised Methods}
\label{Section:SupervisedSemanticRelatednessMethods}
All the semantic relatedness methods described in previous subsections, as well as many other published results not covered here,
can be framed as unsupervised learning techniques, whereby
the semantic relatedness scores emerge from the BK corpus, using some hand-crafted techniques without further human supervision.
There have been a few successful attempts to utilize \emph{supervised} learning techniques as well.
To the best of our knowledge, all of these works follow a similar methodology whereby the features of
a learning instance are assembled from scores obtained by various unsupervised methods (such as those discussed above).
Using this feature generation approach one then resorts to known inductive learning techniques such as
support vector machines (SVMs) \cite{SVM} to learn a classifier or a regressor.

%Strube and Ponzetto
\citeA{WikiRelate2006,PonzettoS07} used SVM regression applied to
instances whose features were constructed as a hybrid of
all the unsupervised techniques described above, which are based on WordNet or Wikipedia.
In addition, Strube and Ponzetto used the Jaccard measure \cite{IntroToIR} applied
to Google search results counts. Overall, their learning instances were comprised of 12 features (six Wikipedia-based scores, five WordNet-based scores
and one Google-based score). They employed a feature selection technique using a genetic algorithm
\shortcite{Mierswa2006},
and applied a standard model selection approach using grid search to identify useful hyper-parameters.
Overall, they obtained 0.66 correlation with the WordSim353 ground truth.

%Bollegala et al.
\shortciteA{2007:MSS:1242572.1242675} considered the \emph{semantic similarity} problem mentioned in Section~\ref{Section:Introduction}.
They constructed a feature vector for a given pair of terms by calculating four well-known co-occurrence measures
(Jaccard, Overlap/Simpson coefficient, Dice coefficient and mutual information) and lexico-syntactic templates
(e.g., `X of Y', `X and Y are', `X of Y'),
which were derived from page counts and snippets retrieved using a web search engine.
Bollegala et al. employed an SVM to classify whether two terms are synonyms or not.
The SVM was trained using examples that were taken from WordNet, considering terms from the same (resp., different)
synset as positive (resp., negative) examples.
The similarity between $t_1$ and $t_2$ was computed as a function of their feature vector's location relative to the SVM decision boundary.

%Agirre et al.
\shortciteA{1620758} considered the binary classification problem of determining which pair among two term pairs is more
related to each other.
In their method,
each instance, consisting of two pairs, $\{t_1,t_2\}$ and $\{t_3,t_4\}$, is represented as a feature vector constructed using
 semantic relatedness scores and ranks from other (unsupervised) relatedness methods.
Specifically, they considered three structure-free semantic relatedness methods and one lexical semantic relatedness method
so that the overall feature vector for an
instance, is a 16-dimensional vector (four scores and four ranks for each term pair).
Using an SVM classifier they obtained 0.78 correlation with WordSim353.
The structure-free BK used for achieving this result consisted of four billion web documents.
They reported that the  overall computation utilized 2000 CPU cores for 15 minutes (approximately 20 days on one core).

Another attempt to utilize SVMs, where features are constructed using unsupervised scores, is reported
by \citeA{Haralambous2011}.
They considered the following unsupervised measures: (i) ESA \cite{GabrilovichM07}; (ii) a weighted shortest path measure based on
WordNet; and (iii) another co-occurrence measure, which is a variant of Jaccard's measure.
Using some combination of these three scores, they managed to achieve 0.7996
correlation with WordSim353. By training an SVM over a training set extracted from WordSim353 term pairs
(represented by these features) they achieved 0.8654 correlation with WordSim353. This is the best
correlation score that was ever reported.
Haralambous and Klyuev noted that this impressive result relies on optimizations
of the ESA hyper-parameters but the precise details of this optimization were not reported.

Both \shortciteA{1620758} and \citeA{Haralambous2011} achieved their
reported results using 10-fold cross validation, thus utilizing 90\% of the available labeled preferences for training.

To summarize, among these works the Agirre et al. approach is the closest to ours, mainly in its formulation of the learning problem.
However, our solution methodology is fundamentally different.

\section{Problem Setup}
\label{Section:ProblemSetup}

We consider a fixed corpus, $\cC \eqdef \{ c_1,c_2,\ldots,c_N \}$, defined to be a set of
contexts. Each \emph{context} $c_i$, $i=1,\ldots,N$, is a textual unit conveying some information in free text.
In this work we consider contexts that are sentences, paragraphs or whole documents.
Let $D \eqdef \{t_1,t_2,\ldots,t_d\}$
be a \emph{dictionary} consisting  of all the terms appearing in the corpus.
A term may be any frequent phrase (unigram, bigram, trigram, etc.) in the corpus, e.g., ``book'', ``New York'',
``The Holly Land.''
Ultimately, our goal is to automatically construct a function $f(t_1,t_2)$ that correctly ranks the relatedness of the terms
$t_1,t_2 \in D$ in accordance with the subjective semantics of a given labler.
We emphasize that we do not require $f$ to provide absolute scores but rather a relative values
inducing a complete order over the relatedness of all terms.

We note that in reality this total order assumption doesn't hold, since the comparison between two term pairs
not sharing any term might be meaningless.
Furthermore, human preferences may contain cycles, perhaps due to comparisons made using different features (as in the rock-paper-scissors game), or due to noise/confusion.
However, we impose total order for simplicity and it reduce the VC-dimension of our hypothesis class (see Section~\ref{Section:ALearningTheoreticalPerspective}).

\subsection{Learning Model}
\label{Section:LearningModel}
Our goal is to construct the function $f$ using supervised learning.
Specifically, the user will be presented with a training set
$\{X_1,\ldots,X_m\}$ to be labeled, where each $X_i \eqdef (\{t^i_1,t^i_2\},\{t^i_3,t^i_4\})$ is a quadruple of terms.
The binary label, $y_i \in \{\pm1\}$, of the instance $X_i$ should be $+1$ if the terms in the first pair $\{t^i_1,t^i_2\}$ are more related to each other than the terms in the second pair $\{t^i_3,t^i_4\}$, and $-1$ otherwise.
Each quadruple along with its label, $(X_i,y_i)$ is also called a \emph{preference}.

Among all possible quadruples, we restrict our attention only to quadruples in the set,
\begin{equation}
\label{eq:Dpref}
D_{pref} \eqdef \left\{  (\{t_1,t_2\},\{t_3,t_4\}) \left| 	
	\begin{array}{ll}
						& t_1,t_2,t_3,t_4 \in D,\\
						& t_1 \neq t_2, t_3 \neq t_4, \\
						& \{t_1,t_2\} \neq \{t_3,t_4\}.
	\end{array}
	\right. \right\}
\end{equation}
The reason to focus only on preferences $X \in D_{pref}$ is that any quadruple $X \in D^4 \setminus D_{pref}$ encodes a meaningless preference,
since the semantic relatedness of term pairs such as $\{t,t\}$ and preferences such as $(\{t_1,t_2\}, \{t_1,t_2\})$ are
trivial.

Denote by $S_m \eqdef \{(X_1,y_1), \ldots, (X_m,y_m)\}$, a set of labeled training examples received from the user.
We assume that if $(X,y) \in {S_m}$ then $(X,-y) \notin {S_m}$.
A binary classifier in our context is a function $h:D_{pref} \to \{\pm1\}$ satisfying, for all
$(\{t_1,t_2\},\{t_3,t_4\}) \in D_{pref}$, the ``anti-symmetry'' condition
\begin{equation}
\label{ref:cond}
h(\{t_1,t_2\},\{t_3,t_4\})=-h(\{t_3,t_4\},\{t_1,t_2\})
\end{equation}
The 0/1 \emph{training error} of $h$ is,
$$
R_m(h) \eqdef \frac{1}{m} \sum_i \I \{ h(X_i) \neq y_i \}.
$$
The standard underlying assumption in supervised learning is that
(labeled) instances are drawn i.i.d. from some unknown distribution $P(X,Y)$ defined over
$D_{pref} \times \{\pm 1\}$.
The classifier $h$ is chosen from some hypothesis class
$\cH$.
In this work we focus on the \emph{realizable setting} whereby
labels are defined by some unknown \emph{target hypothesis} $h^* \in \cH$.
Thus, the underlying distribution reduces to $P(X)$.
The performance of a classifier $h$ is quantified by its
true or (0/1) \emph{test error},
$$
R(h) \eqdef \E_P \{ h(X) \neq Y \}.
$$

\subsection{Learning from Preferences vs. Absolute Scores}
Why do we choose to ask the user about pairwise preferences rather than requesting an absolute relatedness score of a single pair
of terms? Our choice is strongly motivated by recent work showing that answers to such questions are more accurate than
answers to questions about absolute quality.
In order to extract an absolute score, a user
must rely on some implicit global scale, which may or may not exist.
We mention the papers \shortcite{Carterette08hereor,1414238,1718491} as a small sample
of studies that
justify this general approach both theoretically and empirically.

%\subsection{Utilizing this Model in Applications}
%The above model facilitates the implementation of the core subroutine in many applications.
%In a typical scenario we are required to identify the top $k$ most related terms to a given target term $t^*$.
%This can be accomplished by considering a sufficiently large set of term quadruples that are combinations of
%the terms of interest (e.g., most frequent non-stopword terms). On this space of quadruples we will train a model that will later be used to rank
%any given set of pairs including all those that contain the  target term.
%Such comparisons are achieved by projecting the set of quadruples to those of the form $(\{t^*,x\},\{t^*,y\})$, where $x$ and $y$ are any
%terms of interest.

\section{Adaptive Co-occurrence Model}
\label{Section:AdaptiveCoOccurrenceMeasure}

Recognizing the widely accepted idea that the intensity of semantic relatedness between two terms is
a function of their co-occurrence pattern in textual documents,
we would like to somehow measure co-occurrence using a corpus of BK where such patterns are manifested.
Therefore, a major component of the proposed algorithm is an appropriate co-occurrence measure.
However, we also require adaptivity to specific user's subjective relatedness preferences.
Our observation is that such adaptivity can be accomplished by
learning from examples user specific weights to be assigned to contexts, as described blow.
Overall, our approach is to construct
a reasonable initial model, derived only from the BK corpus (without supervision),
which fits a rough general consensus on
relatedness of basic terms. This initial model is the starting point of a learning process that will refine the model to
fit specific user preferences.

In a preliminary study we examined various co-occurrence indices, such as Jaccard measure, pointwise mutual information, KL- and Jensen-Shannon divergences, and latent semantic analysis.
%After comparing these models using our implementation and some results
Based on this study and some published results \cite{Dagan+Lee+Pereira:99a,Terra2003,RecchiaJones2009},
we selected the normalized semantic distance measure of  \citeA{GoogleSimilarityDistance}.\footnote{Note that Cilibrasi and Vitanyi termed this function
``Google similarity distance'' and applied it by relying on Google to retrieve proxies for co-occurrence statistics. In our discussion co-occurrence statistics can be obtained in any desirable manner.}
Specifically, we observed that $\nsd$ by itself can achieve a high 0.745 Spearman correlation with WordSim353
(via our implementation using Wikipedia as the BK corpus)
thus providing a very effective starting point. We note that information measures are
also effective, but not quite as good.\footnote{Pointwise mutual information achieved correlation of
0.73 with WordSim353 \cite{RecchiaJones2009}.}
We also find it appealing that this measure was derived from solid algorithmic complexity principles.

Cilibrasi and Vitanyi defined
the \emph{semantics} $S(t_1, \dots ,t_n)$ of the terms $t_1, \dots ,t_n$,  as the set of all contexts
in which they appear together. Than they defined the \emph{normalized semantic distance ($\nsd$)} between $t_1, t_2$ to be
\begin{equation}
\label{ref:nsd}
\nsd \left( {{t_1},{t_2}} \right) \eqdef \frac{{\max \left\{ {\log \left( {\left| {S\left( {{t_1}} \right)} \right|} \right),\log \left( {\left| {S\left( {{t_2}} \right)} \right|} \right)} \right\} - \log \left( {\left| {S\left( {{t_1},{t_2}} \right)} \right|} \right)}}{{\log \left( Z \right) - \min \left\{ {\log \left( {\left| {S\left( {{t_1}} \right)} \right|} \right),\log \left( {\left| {S\left( {{t_2}} \right)} \right|} \right)} \right\}}},
\end{equation}
where $Z \eqdef \sum_{{t_1},{t_2} \in D} |S(t_1,t_2)|$.

The $\nsd$ function, like any other absolute scoring function for pairs, induces a permutation over
 all the term pairs, and therefore,
can be utilized as a classifier for semantic relatedness preferences, as required.
However, this classifier is constructed blindly without any consideration of the user's subjective preferences.
To incorporate user subjective preferences we introduce a novel extension of $\nsd$ that allows for assigning weights to contexts.
Define the \emph{weighted semantics} $WS(t_1, \dots ,t_n)$ of the terms $t_1, \dots ,t_n$
as
$$
WS(t_1,\ldots,t_n) \eqdef \sum_{c \in S(t_1,\ldots,t_n)} {w(c)},
$$
where $w(c) \in \mathbb{R}^+$ is a weight assigned to the context $c$, where
we impose the normalization constraint
\begin{equation}
\label{ref:WeightConstraint}
\sum_{c \in \C} w(c) = |\C| = N.
\end{equation}
Thus, given a BK corpus, $\cC = \{ c_1,c_2,\ldots,c_N \}$,
and a set $W$ of weights,
$$
W \eqdef \{ w(c_1), w(c_2),\ldots,w(c_N) \},
$$
we define \emph{weighted normalized semantic distance ($\wnsd$)} between $t_1$ and $t_2$ is,

\begin{center}
$\wnsd_{W}(t_1,t_2) \eqdef \frac{\max \{\log (WS(t_1)), \log(WS(t_2))\}-\log (WS(t_1,t_2))}{\log (Z) - \min \{\log (WS(t_1)), \log(WS(t_2))\}}$,
\end{center}
where $Z$ is a normalization constant,
$$
Z \eqdef \sum_{{t_1},{t_2} \in D} WS(t_1,t_2).
$$
We call the set $W$ of weights a \emph{semantic model} and our goal is to learn an appropriate model
from labeled examples.

Recall that our objective is to quantify the relatedness of two terms
in a ``universal'' manner, namely, regardless of the types of relations that link these terms.
Is it really possible to learn a single model $W$ that will encode coherent semantics universally for all terms and
all relations?

At the outset, this objective might appear hard or even impossible to achieve.
%This might be problematic when two different relation types have different co-occur characteristics.
Additional special obstacle is the modeling of synonym relations.
The common wisdom is that synonym terms, which exhibit a very high degree of relatedness,
 are unlikely to occur in the same context \cite<see, e.g.,>{DekangLin657297,BudanitskyH06},
 especially if the context unit is very small (e.g., a sentence).
 Can our model capture similarity relations? We empirically investigate these questions in the
 Sections~\ref{Section:ExpLargeScale}, \ref{Section:ExpMediumScale}, \ref{Section:ExpSmallScale} and \ref{Section:ExpSubjectiveSR} where we
 evaluate the performance of our model via datasets that encompass term pairs with various relations.
 In addition, we further investigate the similarity relations via an ad-hoc experiment in Section~\ref{Section:SemanticSimilarityExp}.

\section{The SemanticSort Algorithm}
\label{Section:SemanticSortAlgorithm}

Let $f_W:D \times D \to \mathbb{R}^+$ be any adaptive co-occurrence measure satisfying the
following properties: (i) each context has an associated weight in $W$;
(ii) $f_W(t_1,t_2)$ monotonically increases with increasing weight(s) of context(s) in $S(t_1,t_2)$;
and (iii) $f_W(t_1,t_2)$ monotonically decreases with (increasing) weight(s) of context(s) in $S(t_1) \setminus S(t_1,t_2)$
or $S(t_2) \setminus S(t_1,t_2)$.

We now present a learning algorithm that can utilize any such function.
We later apply this algorithm while instantiating this function to
$\wnsd$, which clearly satisfies the required properties. Note, however, that many known co-occurrence measures
can be extended (to include weights) and be applied as well.
%We note However, most of the co-occurrence measures in the literature can easily be expanded to satisfies these conditions
%and therefore can also be used as part of $\algname$. }

Relying on $f_W$ we would like utilize empirical risk minimization (ERM)
to learn an appropriate model $W$ of context weights so as to be
consistent with the training set $S_m$.
To this end we designed the following algorithm, called $\algname$,
which minimizes the training error over $S_m$
by fitting appropriate weights to $f_W$.
A pseudocode is provided in Algorithm~\ref{alg:SemanticSort}.

The inputs to $\algname$ are $S_m$, a learning rate factor $\alpha$, a learning rate factor threshold $\alpha_{max}$,
a decrease threshold $\epsilon$,
and a learning rate function $\lambda$.
When a training example is not satisfied, e.g.,
$e=(X=(\{t_1,t_2\},\{t_3,t_4\}),y=+1)$ and $f_{W}(t_1,t_2) < f_{W}(t_3,t_4)$),
we would like to increase the semantic relatedness score of $t_1$ and $t_2$ and decrease the semantic relatedness score of $t_3$ and $t_4$.
$\algname$ achieves this by multiplicatively promoting/demoting the weights of the ``good''/``bad'' contexts in which $t_1,t_2$ and $t_3,t_4$ co-occur.
The weight increase (resp., decrease) depends on $\lambdaUP$ (resp., $\lambdaDN$), which are defined as follows.
\begin{eqnarray*}
\lambdaUP &\eqdef& \frac{\alpha \cdot \lambda(\Delta_{e})+1}{\alpha \cdot \lambda(\Delta_{e})}\\
\lambdaDN &\eqdef& \frac{1}{\lambdaUP}.
\end{eqnarray*}
$\algname$ uses $\lambda$ to update context weights in accordance with the
error magnitude
incurred for example $e=(X=(\{t_1,t_2\},\{t_3,t_4\}),y)$, defined as
$$
\Delta_{e} \eqdef |f_{W}(t_1,t_2) - f_{ W}(t_3,t_4)|.
$$
Thus, we require that $\lambda$ is a monotonically decreasing function so that the greater $\Delta_{e}$ is, the more aggressive $\lambdaUP$ and $\lambdaDN$ will be.
The learning speed of the algorithm depends on these rates, and overly aggressive rates might prevent convergence due to oscillating semantic relatedness scores.
Hence, $\algname$ gradually refines the learning rates as follows.
Define
$$
\Delta  \eqdef \sum_{\mbox{\tiny $e$ is not satisfied}} {{\Delta _e}},
$$
as the total sum of the differences over unsatisfied examples.
We observe that if $\Delta$ decreases at least in $\epsilon$ in each iteration, then
$\algname$ converges and the learning rates remain the same.
Otherwise, $\algname$ will update the learning rate to be less aggressive
by doubling $\alpha$.
Therefore, we require that $0 < \epsilon$.
Note that the decrease of $\Delta$ is only used to control
convergence, but we test $\algname$ using the 0/1 loss function
as described in Section~\ref{Section:Evaluationmethodology}.
$\algname$ iterates over the examples until its hypothesis satisfies all of them,
or $\alpha$ exceeds the $\alpha_{max}$ threshold.
Thus, empirical risk minimization in our context is manifested by minimizing $\Delta$.

\begin{algorithm}
\caption{$\algname(S_m, \alpha, \alpha_{max}, \epsilon, \lambda$)}
%\footnotesize {
\begin{algorithmic}[1]
 \STATE Initialize:
    \STATE $W \leftarrow \overrightarrow 1$
    \STATE $\Delta_{prev} \leftarrow Max Double Value$
 \REPEAT
 \STATE $\Delta \leftarrow 0$
 \FORALL{$e=((\{t_1,t_2\},\{t_3,t_4\}),y) \in S_m$}
    \IF{($y == -1$)}
        \STATE $(\{t_1,t_2\},\{t_3,t_4\}) \leftarrow (\{t_3,t_4\},\{t_1,t_2\})$
    \ENDIF
    \STATE $score_{12} \leftarrow f_{W}(t_1,t_2)$
    \STATE $score_{34} \leftarrow f_{W}(t_3,t_4)$
    \IF {($score_{12} < score_{34})$}
        \STATE \COMMENT{This is an unsatisfied example.}
        \STATE $\lambdaUP \leftarrow \frac{\alpha \cdot \lambda(\Delta_{e}) + 1}{\alpha \cdot \lambda(\Delta_{e})}$
        \STATE $\lambdaDN \leftarrow \frac{1}{\lambdaUP}$
        \STATE $\Delta \leftarrow \Delta + \Delta_{e}$
        \FORALL{$c \in S(t_1,t_2)$}
            \STATE $w(c) \leftarrow w(c) \cdot \lambdaUP$
        \ENDFOR
        \FORALL{$c \in S(t_3,t_4)$}
            \STATE $w(c) \leftarrow w(c) \cdot \lambdaDN$
        \ENDFOR
        \STATE Normalize weights s.t. $\sum_{c \in \cC}w(c)=|\cC|$
    \ENDIF
 \ENDFOR
 \IF{($\Delta - \Delta_{prev} + \epsilon \ge 0$)}
    \STATE $\alpha \leftarrow 2 \cdot \alpha$
    \IF{($\alpha \ge \alpha_{max}$)}
        \RETURN
    \ENDIF
 \ENDIF
 \STATE $\Delta_{prev} \leftarrow \Delta$
 \UNTIL{$\Delta == 0$}
\end{algorithmic}
%}
\label{alg:SemanticSort}
\end{algorithm}

The computational complexity of $\algname$ is as follows. The model $W$ requires $\Theta(|\cC|)$ memory space, since each context is associated with a weight.
In addition, $\algname$ saves a mapping from any $t \in D$ to its $S(t)$. Thus, every occurrence of a term, $t$, in the corpus is represented in this mapping by the index of the relevant
context in $S(t)$.
Let $\nu(t)$ be the number of occurrences of $t \in D$ in the corpus, and define
$|\hbox{corpus}| = \sum_t \nu(t)$. Hence, this mapping
require $\Theta(|\hbox{corpus}|)$ space. Overall, the required space is
$$\Theta(|\cC|+ |\hbox{corpus}|) = \Theta(|\hbox{corpus}|),$$
for learning and classifying.
Our experiments in 64bit Java with 1.3GB filtered Wikipedia (using mainly hash tables) required $\approx$8GB RAM memory.
Due to the normalization constraint~(\ref{ref:WeightConstraint}), when we update a single context's weight, we influence
the weights of all contexts. Therefore,
each update due to unsatisfied example requires $\Theta(|\cC|)$ time complexity.
In the worst case scenario, each iteration requires $\Theta(|S_m| \cdot |\cC|)$. If we denote by $R$ (resp., $r$) the maximum (resp., minimum) semantic relatedness score of $f_W$ to any example in $S_m$,
then the maximum value of $\Delta$ is $(R - r)\cdot |S_m|$. In addition, with the exception of at most
$\log_2(\frac{\alpha_{max}}{\alpha})$ iterations,
$\Delta$ decreases every iteration by at least $\epsilon$.
It follows that
the maximum number of iterations is
$$
\Theta\left(\frac{(R - r)\cdot |S_m|}{\epsilon} + \log_2(\frac{\alpha_{max}}{\alpha})\right).
$$
Thus, the worst case time complexity of $\algname$ is
$$
\Theta\left(\left(\frac{(R - r)\cdot |S_m|}{\epsilon} + \log_2(\frac{\alpha_{max}}{\alpha})\right) \cdot |S_m| \cdot |\cC| \right).
$$
If we implement $\algname$ using $\wnsd$ then the normalization constraint~(\ref{ref:WeightConstraint}) is not necessary.
Let's denote $\alpha$ as the division factor of a certain normalization, $\wnsd_{before}$ as the semantic relatedness score before the normalization,
$\wnsd_{after}$ as the semantic relatedness score after the normalization,
$$
\max \{\log (WS(t_1)), \log(WS(t_2))\} \eqdef \log(WS(t_{max})),
$$
and
$$
\min \{\log (WS(t_1)), \log(WS(t_2))\} \eqdef \log(WS(t_{min})).
$$
We thus have,
\begin{eqnarray*}
\wnsd_{before} & \eqdef & \frac{\log(WS(t_{max})) - \log (WS(t_1,t_2))}{\log (Z) - \log(WS(t_{min}))} \\
& = & \frac{\left(\log(WS(t_{max})) - \log(\alpha)\right)- \left(\log (WS(t_1,t_2)) - \log(\alpha)\right)}{\left(\log (Z) - \log(\alpha)\right) - \left(\log(WS(t_{min})) - \log(\alpha)\right)} \\
& = & \frac{\log\left(\frac{WS(t_{max})}{\alpha}\right) - \log\left(\frac{WS(t_1,t_2)}{\alpha}\right)}{\log\left(\frac{Z}{\alpha}\right) - \log\left(\frac{WS(t_{min})}{\alpha}\right)} \\
& \eqdef & \wnsd_{after}.
\end{eqnarray*}
Hence, the worst case time complexity of $\algname$ using $\wnsd$ is
$$
\Theta\left(\left(\frac{(R - r)\cdot |S_m|}{\epsilon} + \log_2(\frac{\alpha_{max}}{\alpha})\right) \cdot |S_m| \right).
$$
We emphasize that this is a worst case analysis.
In practice, the precise time complexity is mainly dependent on
the number of training errors.
Assuming that computing $f_{W}(t_1,t_2)$ depends only on $S(t_1)$ and $S(t_2)$,
this computation requires $\Theta(|S(t_1) \cup S(t_2)|)$ time complexity.
Thus, classifying an instance $(\{t_1,t_2\}, \{t_3,t_4\})$ requires $\Theta(|S(t_1) \cup S(t_2) \cup S(t_3) \cup S(t_4)|)$ time.
If we denote by $M$ the total number of unsatisfied examples encountered by $\algname$
during training, and assuming that in our BK corpus, $S(t) \ll |\cC|$ for every term, then the overall time complexity of the learning process is $\Theta(M \cdot |\cC|)$
($\Theta(M)$ using $\wnsd$), since the overall time required to process satisfied instances is negligible.
Finally, we note that in all our experiments the total number of iterations was at most 100, and it was always the case
that $M < |S_m|$.

\section{Empirical Evaluation}
\label{Section:EmpiricalEvaluation}
To evaluate the effectiveness of $\algname$ we conducted several experiments. One of the barriers in designing these experiments is the lack of
labeled dataset of term quadruples as required by our model.  The common benchmark datasets are
attractive because they were labeled by human annotators, but these datasets are rather small.
When considering a small real world application involving even 500
vocabulary terms, we need to be able to compare the relatedness of many
of the $\binom{500}{2} = 124,750$ involved pairs. However, the largest available dataset, WordSim353, contains only $353$
pairs\footnote{In effect there are 351 pairs since each of the
pairs \texttt{money -- bank} and \texttt{money -- cash} appear twice, with two different scores. In our experiments we simply merged them and used average scores.}
over its $434$ unique vocabulary terms.

\subsection{The GSS Dataset}
Although we utilized all available datasets in our experiments (see below), we sought a benchmark of significantly larger scale in order to approach
real world scenarios. In such scenarios where the vocabulary is large our resources limit us to train $\algname$ only on negligible fraction from the available preferences
(the largest fraction is about $10^{-5}$). As opposed to these available humanly annotated, where we examined $\algname$ ability to learn human preferences,
the larger dataset has a different objective: Verify if learning can be achieved while leveraging such a tiny statistical fraction of the dataset.
Furthermore, we want this large dataset to still be positive correlated to human semantic preferences as a sanity check.

Without access to a humanely annotated dataset of a large scale, we synthesized a labeled dataset as follows.
Noting that a vocabulary of 1000-2000 words covers about 72\%-80\% of written English texts \cite{francis1982}, we
can envision practical applications involving vocabularies of such sizes.
We therefore selected a dictionary $D_n$ consisting of the $n$ most frequent English words ($n = 500, 1000$).
For each of the $\binom{n}{2}$ term pairs over $D_n$ we used an independent corpus of English texts, namely the Gutenberg Project,
 to define the semantic relatedness score of pairs, using the $\nsd$ method, applied with sentence based contexts.
We call this scoring method the \emph{Gutenberg Semantics Score (GSS)}.

Project Gutenberg is a growing repository that gathers many high quality and classic literature that is
freely available on the web. For example, among the books one can find
\texttt{Alice's Adventures in Wonderland}, \texttt{The Art of War}, \texttt{The Time Machine}, \texttt{Gulliver's Travels}, and many well
known fiction ebooks.
Currently, Project Gutenberg offers over 36,000 ebooks.\footnote{These ebooks appear in many formats such as
HTML, EPUB, Kindle, PDF, Plucker, free text, etc.}
In this work we used a complete older version of Project Gutenberg from February 1999 containing only 1533 texts
bundled by Walnut Creek CDROM. We didn't try to use any other version and we used this old and small version merely because it was in our possession
and it served the purpose of our experiments as mention above. We believe that any version can be utilized as there is no problem in $\algname$ which prevents us
from using any different version or other textual corpus?

While GSS is certainly not as reliable as human generated score (for the purpose of predicting human scores),
we show below that GSS is positively correlated with human annotation, achieving 0.58 Spearman correlation with the WordSim353 benchmark.
Given a set of term pairs together with their semantic relatedness scores (such as those generated by GSS), we construct a labeled set of preferences
according to semantic relatedness scores (see definitions in Section~\ref{Section:ProblemSetup}).

We emphasize that the texts of the Project Gutenberg were taken conclusively and as is, without any modifications, to avoid
any selection bias.\footnote{The GSS dataset will be made publicly available.}
Nevertheless, despite its statistical correlation to human annotation,
our main objective is not to evaluate absolute performance scores,
but rather to see if generalization can be accomplished by $\algname$ at this scale, and in particular,
 with an extremely small fraction of the available training examples.

\subsection{Background Knowledge Corpora}
An integral part of the $\algname$ model is its BK corpus.
We conducted experiments using two corpora.
The first corpus is the snapshot of Wikipedia from 05/11/05 preprocessed using Wikiprep.\footnote{Wikiprep is an XML preprocessor for Wikipedia,
available at \url{http://www.cs.technion.ac.il/~gabr/resources/code/wikiprep}.} We used this old version of Wikipedia only because it was already available preprocessed and,
as mention in the previous section, we saw no importance of choosing one version over the other as anyone will do.\footnote{Wikipedia's dump is contains many macros that need to be processed in order to achieve the raw text.}
Following \cite{GabrilovichM07}, in order to remove small and overly specific articles, we filtered out articles containing
either less than 100  non-stopword terms and/or less than 5 incoming links and/or less than 5 outgoing links.
The second corpus we used is the Project Gutenberg mentioned above.
We emphasize that in all experiments involving GSS scores only Wikipedia was used as the BK corpus.
Also, in each experiment we either used Wikipedia or Gutenberg as a BK corpus and not both.
In all the experiments we ignored stopwords and stemmed the terms using Porter's stemmer.\footnote{Porter's stemmer is available at
\url{http://tartarus.org/~martin/PorterStemmer}.}
Finally, We considered three types of contexts: sentences, paragraphs and whole documents.
Sentences are parsed using `.' as a separator without any further syntax considerations;
paragraphs are parsed using an empty line as a separator.
No other preprocessing, filtering  or optimizations were conducted.
After some tuning, we applied $\algname$ with the following hyper-parameters that gave us the best result:  $\alpha = 1$, $\alpha_{max} = 32$, $\epsilon = 0.0001$, and\footnote{Our brief attempts with various continuous functions (linear or exponential) were not as successful. Thus, we used them because they provided the best performance.}
$$
\lambda(\Delta_e) = \left\{
             \begin{array}{ll}
               4, & \hbox{if $\Delta_e \ge 0.1$;} \\
               8, & \hbox{if $0.1 > \Delta_e \ge 0.04$;} \\
               16, & \hbox{if $0.04 > \Delta_e \ge 0.005$;} \\
               32, & \hbox{otherwise.}
             \end{array}
           \right.
$$

\subsection{Evaluation Methodology}
\label{Section:Evaluationmethodology}
Consider a collection $P$ of preferences, where  each preference is a quadruple, as define in Section~\ref{Section:ProblemSetup}.
When we evaluate performance of the algorithm w.r.t. a training set of size $m$, we choose an $m$-subset,
$S_m \subseteq P$ uniformly at random.
The rest of the preferences in $P \setminus S_m$ are taken as the test set.\footnote{Formally speaking, this type of sampling
without replacement of the training set, is within a standard transductive learning model \cite[Sec. 8.1,Setting 1]{Vapnik98} .}
However, if $P \setminus S_m$ remains very large, only 1,000,000 preferences, chosen uniformly at random from $P \setminus S_m$, are
taken for testing.
The training set $S_m$ is fed to $\algname$. The output of the algorithm is an hypothesis $h$, consisting of a weight vector $W$
that includes a component for each context in $\cC$.
Then we apply the hypothesis on the test set and calculate the resulting
accuracy (using the 0/1 loss function). This quantity provides a relatively accurate estimate of (one minus) the true error $R(h)$.
In order to obtain a learning curve we repeat this evaluation procedure for a monotonically increasing sequence of training set sizes.
The popular performance measure in semantic relatedness research is the Spearman correlation coefficient of the ranking obtained by
the method to the ground truth ranking. Therefore, we also calculated and reported it as well.
In addition, in order to gain statistical confidence in our results, we repeated the experiments for each
training set size multiple times and reported the average results.
For each estimated average quantity along the leaning curve we also calculated
its standard error of the mean (SEM), and depicted the resulting SEM values as error bars.

\subsection{Experiment 1: large scale}
\label{Section:ExpLargeScale}
\begin{figure}
\centerline{
\psfig{figure=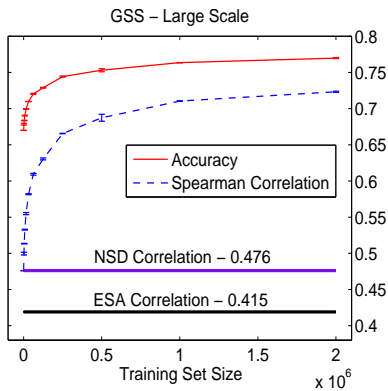, height=2in, width=2in}
}
\caption{Experiment 1 (large scale) - Learning curves for test accuracy (solid) and test correlation (dashed), with standard error bars.
Lower horizontal line at 0.415 marks the performance of ESA \cite{Gabrilovich2005}.
Upper horizontal line at 0.476 marks the performance of $\nsd$ \cite{GoogleSimilarityDistance}.}
\label{Figure:GSSLargeScale}
\end{figure}
In order to evaluate $\algname$ on ambitious, large scale and quite realistic scenario, we conducted the following experiments.
Taking $D_{1000}$ (the top 1000 most frequent terms in Wikipedia) we considered all possible preferences.
Note that the number of preferences associated with $D_{1000}$ is huge, containing
about $10^{12}/4$ quadruples. We labeled the preferences according to GSS as described above.
In generating the learning curve we were only able to reach $m = 2,000,000$ training examples, thus utilizing
an extremely small fraction of the available preferences (the largest fraction is about $10^{-5}$).
%The computation time was about two weeks on a high end server with eight cores.
Figure~\ref{Figure:GSSLargeScale} presents 0/1 test accuracy and Spearman correlation learning curves.
On this figure we also mark the results obtained by two unsupervised methods: (i) $\nsd$ using Wikipedia as BK corpus
 with paragraph level contexts;
(ii) the well known ESA method using the same filtered Wikipedia snapshot mentioned above.
Both these unsupervised performance scores were calculated by us using our implementations of these methods.
%These learning curves (as well as all other learning curves in other figures) were generated by increasing the training set percentage of the entire set of preferences.
%Each point on
%each curve represents the \emph{test performance} averaged over of several independent random experiments and standard error bars are shown.
It is evident that $\algname$ successfully generalized the training sample and accomplished a notable improvement over its starting point.
We believe that these results can serve as a proof of concept and
confirm $\algname$'s ability to handle real world challenges.

\subsection{Experiment 2: medium scale}
\label{Section:ExpMediumScale}
We repeated the previous experiment now with $D_{500}$, taken to be subset of 500 terms from $D_{1000}$ chosen uniformly at random.
All other experiment parameters were precisely as in Experiment 1. The resulting learning curves are shown in
Figure~\ref{Figure:GSSMediumScale}.
Clearly, this medium scale problem gave rise to significantly higher absolute performance scores.
We believe that the main reason for this improvement (with respect to the large scale experiment)
is that with $D_{500}$ we were able to utilize a larger fraction of preferences in training.

\begin{figure}[h!tb]
\centerline{
\psfig{figure=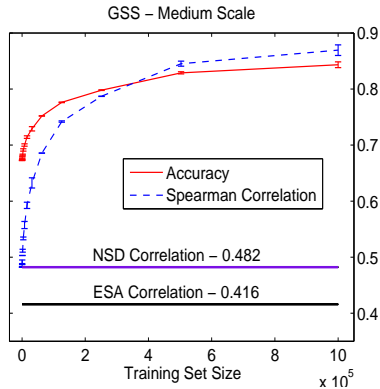 , height=2in, width=2in}
}
\caption{Experiment 2 (medium scale) - Learning curves for test accuracy (solid) and test correlation (dashed),with standard error bars.
Lower horizontal line at 0.416 marks the performance of ESA \cite{Gabrilovich2005}.
Upper horizontal line at 0.482 marks the performance of $\nsd$ \cite{GoogleSimilarityDistance}.}
\label{Figure:GSSMediumScale}
\end{figure}

\subsection{Experiment 3: small scale}
\label{Section:ExpSmallScale}

As mentioned in Section~\ref{Section:RelatedWork}, most of the top performing known techniques,
including the reported supervised methods, evaluated  performance with respect to the WordSim353 benchmark.
In order to link the proposed approach to the current literature
we also conducted an experiment using WordSim353 as a source for labeled preferences.
This experiment serves three purposes.
First, it can be viewed as a sanity check for our
method, now challenging it with humanly annotated scores.
Second, it is interesting to examine the performance advantage of our supervised approach vs. no systematic supervision
as obtained by the unsupervised methods (we already observed in Experiments 1\&2 that our supervised method
can improve the scores obtained by ESA and $\nsd$).
Finally, using this experiment we are able compare between $\algname$ and the other known supervised methods that so far have been relying
on SVMs.

\begin{figure}
\centerline{
\psfig{figure=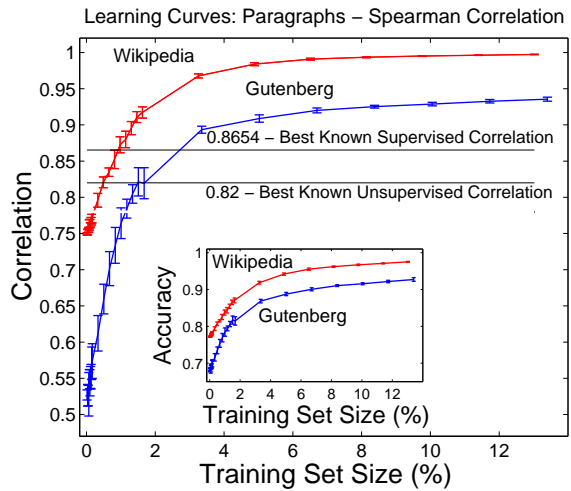, width=3in }
%\psfig{figure=ParagraphContext-Both,height=12cm}
%, height=2in, width=2in
}
\caption{Experiment 3 (small scale) - Learning curves for test correlation and test accuracy, with standard error bars using
either Wikipedia or Gutenberg. Lower horizontal line at 0.82 marks the best known unsupervised result for WordSim353 \cite{RadinskyAGM11}.
Upper horizontal line at 0.8654 marks the best known supervised result for WordSim353 \cite{Haralambous2011}.}
\label{Figure:ParagraphContext-Both}
\end{figure}

Figure~\ref{Figure:ParagraphContext-Both} shows the learning curves obtained by $\algname$ applied with paragraph contexts
using either Wikipedia or Gutenberg (but not both together) as a BK corpus.
The lower horizontal line, at the 0.82 level, marks the best known \emph{unsupervised} result obtained for WordSim353 \cite{RadinskyAGM11}.
The upper horizontal line, at the 0.8654 level, marks the best known \emph{supervised} result \cite{Haralambous2011}.
It is evident that quite rapid learning is accomplished using either the Wikipedia or the Gutenberg models, but Wikipedia enables significantly faster learning and smaller sample complexity for each error level. The curves in the internal panel show the corresponding
test \emph{accuracies} (0/1 loss) for the same experiments.
Note that meaningful comparisons between $\algname$ and the other (SVM based) supervised methods (described in Section~\ref{Section:SupervisedSemanticRelatednessMethods})
can  only be made when considering the same train/test partition sizes.
Unlike our experimental setting, both \shortciteA{1620758} and \citeA{Haralambous2011} achieved their
reported results (0.78 and 0.8654 correlation with WordSim353, respectively) using 10-fold cross validation, thus utilizing
90\% of the available labeled preferences for training.
When considering only the best results obtained at the top of the learning curve, $\algname$ outperforms
the best reported supervised performance after consuming 1.5\% of all the available WordSim353 preferences using the Wikipedia model and
after consuming 3\% of the preferences using the Project Gutenberg model.

\begin{figure}[htb]
\centerline{
\psfig{figure=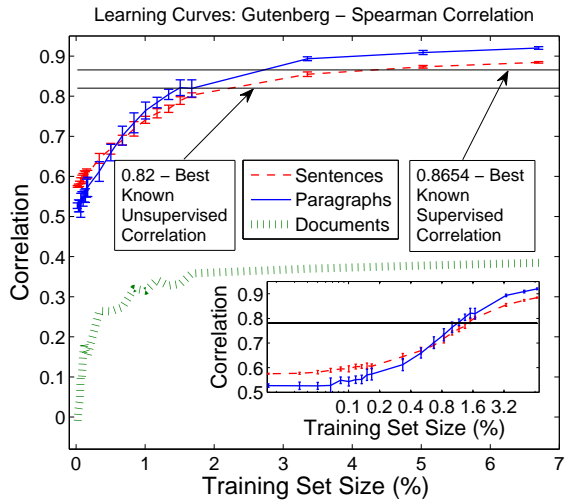,width=3in}
}
\caption{Experiment 3 (small scale) - Learning curves for test correlation with standard error bars using
Project Gutenberg applied with sentences, paragraphs and whole document as context types.
Lower horizontal line at 0.82 marks the best known unsupervised result for WordSim353 \cite{RadinskyAGM11}.
Upper horizontal line at 0.8654 marks the best known supervised result for WordSim353 \cite{Haralambous2011}.
The internal panel zooms into the same curves of sentence- and paragraph-based semantic relatedness, now with logarithmic $X$-axis.}
\label{Figure:Gutenberg-Spearman}
\end{figure}

Figure~\ref{Figure:Gutenberg-Spearman} depicts three Gutenberg learning curves: one for each context type.
The internal panel zooms into the same curves of sentence- and paragraph-based contexts, now with logarithmically scaled
$X$-axis to emphasize their differences.
As before, the lower (resp., upper) horizontal line at 0.82 (resp., 0.8654) marks the best known unsupervised (resp., supervised) result for WordSim353 \cite{RadinskyAGM11} \cite<resp.,>{Haralambous2011}.
Clearly, paragraph contexts exhibit the best test performance for almost all training set sizes.
In contrast, contexts consisting of whole documents perform poorly, to the extent that even after utilizing the largest training set size,
they are still way behind sentences and paragraphs (even without using a single labeled example).
A similar comparison (not presented) for Wikipedia contexts showed entirely different picture with all contexts exhibiting very similar
(and almost indistinguishable) performance as shown for paragraphs in Figure~\ref{Figure:ParagraphContext-Both}.

\subsection{Experiment 4: subjective semantic relatedness}
\label{Section:ExpSubjectiveSR}
To examine the ability of $\algname$ to adapt to subjective semantic relatedness ranking,
we created two new synthetic sets of semantic relatedness scores to all WordSim353 pairs:
\begin{enumerate}
\item[(i)] A \emph{Wikipedia set} of scores that was calculated using paragraph-based $\nsd$ over Wikipedia;
\item[(ii)] A \emph{Gutenberg set} that  was generated using paragraph-based $\nsd$ over the Gutenberg corpus.
\end{enumerate}
We consider these two sets as proxies for two different ``subjective'' semantic relatedness preferences.\footnote{Indeed, these two sets exhibited numerous significantly different semantic relatedness valuations. For example,
\texttt{nature} and \texttt{environment} received high score in Wikipedia but very low score in Gutenberg,
 and \texttt{psychologist} and \texttt{fear} were much more similar in Gutenberg than in Wikipedia.}
Table~\ref{Table:SyntheticExp} outlines two learning curves: the first corresponds to
learning the Gutenberg preferences using Wikipedia as the BK corpus, and the second, for learning
the Wikipedia preferences using Gutenberg as BK.
It is evident that in both cases $\algname$ successfully adapted to these subjective preferences
achieving excellent test performance in both cases.
  \begin{table}
  \caption{Experiment 4 (subjective semantic relatedness) - Spearman Correlation.}
\label{Table:SyntheticExp}
\begin{center}
    \begin{tabular}{ | c || c | c | c | c | c | c |}
  \hline
  Training test size (\%) & 0 & 0.5 & 1 & 2 & 4 & 8\\
  \hline
  \hline
  Wiki learns Gutenberg & 0.65 & 0.77 & 0.85 & 0.93 & 0.97 & 0.99\\
  Gutenberg learns Wiki & 0.62 & 0.73 & 0.82 & 0.89 & 0.94 & 0.96\\
  \hline
    \end{tabular}
\end{center}
    \end{table}

\subsection{Experiment 5: semantic similarity}
\label{Section:SemanticSimilarityExp}
Synonymous relations are considered among the most prominent semantic relations. \emph{Semantic similarity} is a sub-domain
of semantic relatedness where one attempts to assess the strength of synonymous relations.
A widely accepted approach to handle synonyms (and antonyms) is via distributional similarity \cite{DekangLin657297,BudanitskyH06}.
In this approach, to determine the similarity of terms $t_1$ and $t_2$ we consider $D(t_1)$ and $D(t_2)$, the ``typical''
distributions of terms in close proximity to $t_1$ and $t_2$, respectively. It is well known that these distributions tend to resemble
whenever $t_1$ is similar to $t_2$, and vice versa.
In contrast,  $\algname$ computes its similarity scores based on co-occurrence counts, and the conventional wisdom is
that synonyms tend not to co-occur. A natural question then is how well and in what way can  $\algname$  handle synonymous relations.

In this section we examine and analyze the behavior of $\algname$ on a specialized semantic similarity task.
To this end, we use the semantic similarity datasets, namely R\&G and M\&C, which are introduced and described
in Section~\ref{Section:StandardBenchmarkDatasets}.

\begin{figure}[htb]
\centerline{
\psfig{figure=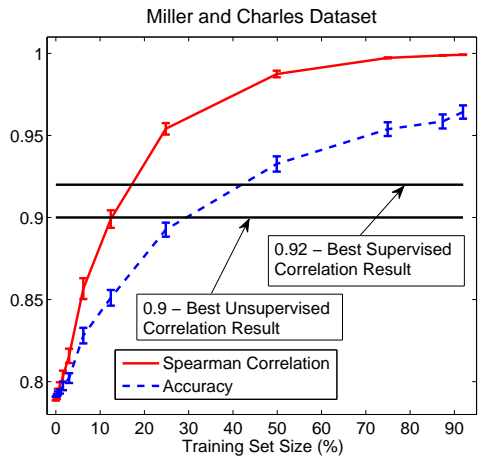, width=2.5in}
}
\caption{Experiment 5 (semantic similarity with Miller \& Charles dataset) - Learning curves for test correlation (solid)
and test accuracy (dashed) with standard error bars.
Lower horizontal line at 0.9 marks the best known unsupervised results \cite{TKDE2003,HughesR07}.
Upper horizontal line at 0.92 marks the best known supervised result \cite{1620758}.
}
\label{Figure:MillerCharlesDatasetResults}
\end{figure}

\begin{figure}[htb]
\centerline{
\psfig{figure=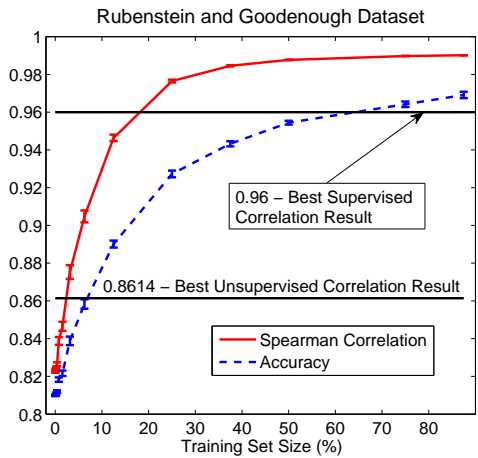, width=2.5in}
}
\caption{Experiment 5 (semantic similarity with Rubenstein and Goodenough dataset) - Learning curves for
test correlation (solid) and test accuracy (dashed) with standard error bars.
Lower horizontal line at 0.8614 marks the best known unsupervised result \cite{TsatsaronisVV10}.
Upper horizontal line at 0.96 marks the best known supervised result \cite{1620758}.
}
\label{Figure:GoodenoughDatasetResults}
\end{figure}

Figure~\ref{Figure:MillerCharlesDatasetResults} depicts the results obtained for the M\&C dataset.
The lower horizontal line, at the 0.9 level, marks the best known \emph{unsupervised} results obtained for M\&C dataset \cite{TKDE2003,HughesR07}.
The upper horizontal line, at the 0.92 level, marks the best known \emph{supervised} result obtained for M\&C dataset \cite{1620758}.
Figure~\ref{Figure:GoodenoughDatasetResults} depicts the results obtained for the R\&G dataset.
The lower horizontal line, at the 0.8614 level, marks the best known \emph{unsupervised} results obtained for R\&G dataset \cite{TsatsaronisVV10}.
The upper horizontal line, at the 0.96 level, marks the best known \emph{supervised} result obtained for R\&G dataset \cite{1620758}.
The learning curves depicted in both figures clearly indicate that learning synonyms using our method is an achievable task, and in fact,
can improve upon the distributional similarity methods.
While synonyms and antonyms co-occur infrequently, they still do co-occur.
It is a nice property of our model that it can leverage these sparse co-occurrence counts
and accurately detect synonyms by sufficiently increasing the weights of their mutual contexts.

\section{Model Interpretability}
The semantic model learned by $\algname$ is encoded in its weight vector $W$.
In this section we summarize our initial study to explore the model $W$ and gain some insight into its structure.
Are the weights in $W$ ``arbitrarily'' optimized to reduce the training error, or is it the case that they are
organized in a meaningful and interpretable manner?
Can we learn from $W$ something about the human rater(s) who tagged the training set?
Can we say something about their world knowledge and/or intellectual interests?

Trying to answer the above questions we conducted the following preliminary study.
While the results we obtained are not sufficient for fully answering the above questions, they are indicative and suggest
that the semantic model $W$ contains useful information that can be interpreted and perhaps even be utilized in applications.
In our experiments, due to the absence of human annotating resources,
we again synthesized a ``human rater'' whose knowledge is focused on a specific topic.

Given a specific topic $T$ in Wikipedia (e.g., sports) we extracted the set $S_T$ of documents pertaining to $T$ (using the Wikipedia
topic tags), and partitioned $S_T$ uniformly at random into two subsets, $S^1_T$ and $S^2_T$.
The subset $S^1_T$ was used for labeling, and $S^2_T$ was used as part of the BK corpus together with the rest of
the Wikipedia corpus.
Our synthetic rater annotated preferences based on $\nsd$ applied over $S^1_T$, whose articles were partitioned to paragraph units.
We call the resulting semantic preferences the \emph{$T$-semantics}.

Taking $D_{1000}$ as a dictionary, we generated a training set by sampling uniformly at random
$m = 2,000,000$ preferences, which were tagged using the  $T$-semantics.
We then applied $\algname$ to learn the $T$-semantics using this training set
while utilizing $S^2_T$ (as well as the rest of Wikipedia) as a BK corpus, whose documents were parsed
to the paragraph level as well. We then examined the resulting $W_T$ model.

\begin{table}
\begin{center}
 \resizebox{12cm}{!} {
\begin{tabular}{ | c || c | c || c | c || c | c || c | c |}
  \hline
  \multirow{2}{*}{ \# } & \multicolumn{2}{c||}{play} &  \multicolumn{2}{c||}{player} & \multicolumn{2}{c||}{record} &\multicolumn{2}{c|}{club}\\
  \cline{2-9}
  & Music & Sports & Music & Sports & Music & Sports & Music & Sports \\
  \hline
  1 & band & game & instrument & play & release & set & dance & football \\
%  \hline
  2 & guitar & team & play & league & album & season & night & league \\
%  \hline
  3 & instrument & season & replace & game & label & win & heart & cup \\
%  \hline
  4 & perform & player & join & season & band & career & fan & play \\
%  \hline
  5 & time & football & guitar & born & song & finish & local & divis \\
%  \hline
  6 & year & first & technique & team & first & run & house & season \\
%  \hline
  7 & role & score & key & football & new & game & London & manage \\
%  \hline
  8 & tour & club & example & professional & studio & won & scene & success \\
%  \hline
  9 & two & year & football & baseball & production & score & mix & found \\
%  \hline
  10 & new & career & hand & major & sign & second & radio & player \\
  \hline
\end{tabular}
}
\end{center}
\caption{Model Interpretability - Top 10 related terms according to Music and Sports Semantics.}
\label{Table:MusicSportSemantics}
\end{table}

Two topics $T$ were considered: Music and Sports, resulting in two models:
$W_{music}$ and $W_{sports}$.
In order to observe and understand the differences between these two models,
we identified and selected, before the experiment, a few target terms that have ambiguous meanings with respect to Music and Sports.
The target terms are:
\begin{quote}
\begin{center}
\texttt{play}, \ \ \texttt{player}, \ \ \texttt{record}, \ \ \texttt{club}.
\end{center}
\end{quote}
Table~\ref{Table:MusicSportSemantics} exhibits the top 10 most related terms to each of the target terms according
to either $W_{music}$ or $W_{sports}$.
It is evident that the semantics portrayed by these lists are quite different and nicely
represent their topics as we may intuitively expect.
The table also emphasizes the inherent subjectivity in semantic relatedness analyses, that should be accounted for
when generating semantic models.

Given a topical category $C$ in wikipedia, and a hypothesis $h$, we define
the \emph{aggregate $C$-weight according to $h$},
to be the sum of the weights of all contexts that belong to an article that is categorized into $C$ or its Wikipedia
sub-categories.
Also, given a category $C$, we denote by $h_{init}^C$, its initial hypothesis and by $h_{final}^C$, its final hypothesis (after learning).\footnote{The initial
hypotheses vary between topics if their respective BK corpora are different.}
In order to evaluate the influence of the labeling semantics on $h_{final}^C$,
we calculated, for each topic $T$ the difference between its aggregate $C$-weight according
to $h_{init}^C$ and according to $h_{final}^C$.

\begin{figure}[htb]
\centerline{
\psfig{figure=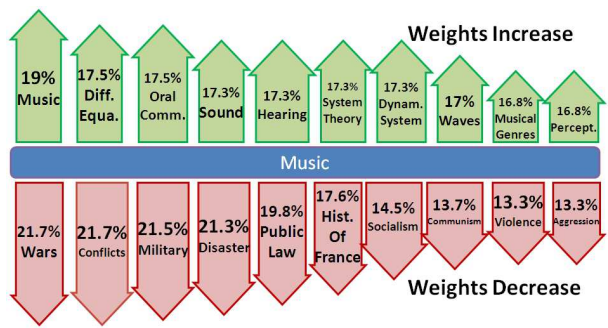,width=8cm}
}
\caption{Model Interpretability - Weights increase (upper/green) and decrease (lower/red) of Wikipedia's major categories according to Music hypotheses.}
\label{Figure:MusicWeightsChange}
\end{figure}[htb]
\begin{figure}
\centerline{
\psfig{figure=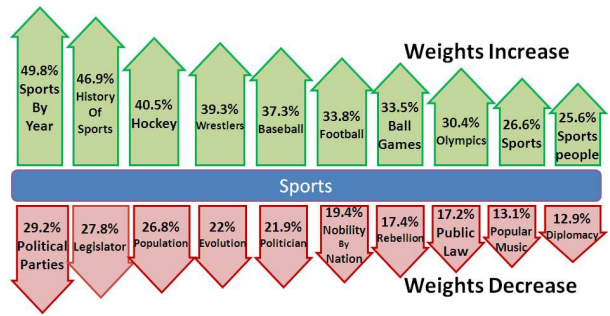,width=8cm}
}
\caption{Model Interpretability - Weights increase (upper/green) and decrease (lower/red) of Wikipedia's major categories according to Sports hypotheses.}
\label{Figure:SportsWeightsChange}
\end{figure}

Figures~\ref{Figure:MusicWeightsChange}~and~\ref{Figure:SportsWeightsChange} present the increase/decrease in those aggregate $C$-weights for
Wikipedia's major categories $C$.
In both cases of labeling topics, Music or Sports, it is easy to see that, by and large, the aggregate weights of
categories  that are related to the labeling topic were increased,
while weights of unrelated categories were decreased.
Surprisingly, when considering the Music topic, many mathematical categories
dramatically increased their weight.\footnote{Indeed, both Music and Mathematics share a large vocabulary.
Furthermore, it is  common wisdom that successful mathematicians are often also accomplished musicians and vice versa.}
To summarize, it is clear that $\algname$ successfully
identified the intellectual affiliation of the synthesized labeler.

While these results aren't conclusive (and can be viewed as merely anecdotal), we believe that
they do indicate that the automatically emerging weights in the model $W$ are organized
in a meaningful and interpretable manner, which encodes the labeling semantics as a particular weight distribution
over the corpus topics.
%We believe that the resulting hypothesis maps the ``labelers knowledge and experience'' into distribution over the corpus's topics.
In addition, not only did $\algname$ identify the labeler BK, it also unexpectedly revealed related topics.

\section{A Learning-Theoretic Perspective}
\label{Section:ALearningTheoreticalPerspective}

Here we would like to present some initial thoughts on the learnability of semantic relatedness.
Classic learning-theoretic considerations ensure that generalization will be achieved if the hypothesis class
$\cH$ will be sufficiently expressive to allow fitting of the training set, but still appropriately restricted
to avoid overfitting. Appropriate fitting is of course a function of the hypothesis class expressiveness
and the training sample
size.
Assuming a realizable (noise-free) setting, a classical result in statistical learning
theory is that
any consistent learning algorithm
(that perfectly fits the training set) will require a sample complexity of,
$$
O\left(\frac{d}{\eps}\log\left(\frac{1}{\eps}\right) + \frac{1}{\eps}\log\left(\frac{1}{\delta}\right)\right),
$$
to achieve error $\eps$ with probability $1-\delta$ over random choices of the training set.
Here, $d = VCdim (\cH)$, is the VC-dimension of $\cH$ (see, e.g., \cite{AnthoB99}).
Conversely, it has been shown (for particular worst case distribution and hypothesis class) that,
$$
\Omega\left(\frac{d}{\eps} + \frac{1}{\eps}\log\left(\frac{1}{\delta}\right)\right)
$$
examples are necessary.
Thus, the VC-dimension is a necessary and sufficient dominating factor that will determine the required training sample size
if we seek a distribution free bound.

We now show that in our context, a completely unrestricted hypothesis class, $\cH_{all}$, whose hypotheses only
satisfy the ``anti-symmetry'' condition (\ref{ref:cond}),
is completely useless,
because its VC-dimension is
$\binom{ \binom{|D|}{2} }{2} = \Theta \left(|D|^4 \right)$.
Therefore, using $\cH_{all}$ is of course a triviality because the number of quadruples in $D_{pref}$
is exactly $2 \cdot \binom{ \binom{|D|}{2} }{2}$ so there is nothing that could be gained by sampling (the proof is provided in Appendix A).

We now consider the hypothesis class $\cH_{\pi}$ of permutations over term pairs.
Each hypothesis in this class is in essence a
full order over the pairs.
It is not hard to prove that the VC-dimension of this class is $\binom{|D|}{2}-1 = \Theta(|D|^2)$
(see the proof in Appendix B).

The set of permutations, with its $\Theta(|D|^2)$ VC-dimension, provides a substantial improvement over the set of all (anti-symmetric)
hypotheses. However, this dimension is still quite large, and requires huge resources
for gathering sufficiently large training sets.
In contrast, we already observed the ability of $\algname$ to learn semantic relatedness preferences quite well
with relatively small training sets.
Can this be explained using VC dimension arguments?

While currently we don't know how to explicitly evaluate the capacity of the hypothesis space induced by $\algname$,
we observe that the use of a BK corpus through the $\nsd$ measure, provides further capacity reductions
by placing many constraints on the set of allowable permutations.
For example, observe that $\algname$ only updates the weights of contexts that include both terms in a given pair,
so it cannot change the semantic relatedness score of terms that do not co-occur; hence, the relative order of unrelated terms is predetermined.
Moreover, considering the structure of the $\nsd$ measure we know
the following lower bound on the semantic relatedness score\footnote{The higher the score is the less the terms are related.}
of two terms, $t_1$ and $t_2$,
$$
\frac{\max \{\log (WS(t_1)), \log(WS(t_2))\}}{\log (Z) - \min \{\log (WS(t_1)), \log(WS(t_2))\}}.
$$
It follows that $\algname$ has limited freedom in reducing semantic relatedness scores.

Finally, $\algname$ regularizes context's weights by normalizing the total sum of the weights.
Therefore, an update of a context's weight influences the weight of all the contexts,
which, in turn, influences the semantic relatedness score of
all the term pairs in general, and specifically, all the term pairs that contain terms within this context.
This mutual dependency was especially evident in the large and medium experiments
(Sections~\ref{Section:ExpLargeScale}~and~\ref{Section:ExpMediumScale}) where the learning complexity was higher.

Such considerations including other statistical and graph-theoretic properties of a particular BK corpus
(viewed as a weighted graph whose nodes are terms or term pairs),
can in principle be used in attempts to estimate the effective VC-dimension implied by $\algname$.
We believe that such considerations and analyses are important as they
can lead to better understanding and improvements of the learning process and perhaps even help in characterizing the
role and usefulness of particular BK corpora.

\section{Concluding Remarks}
Building on successful and interesting ideas, we presented in this paper a novel supervised
method for learning semantic relatedness. The proposed algorithm exhibits interesting performance over a large and medium scale
problems and excellent performance on small scale problems.
In particular, it significantly outperforms the best \emph{supervised} semantic relatedness method. Perhaps expectedly,
our test scores are also distinctly superior to
scores obtained by a plethora of unsupervised semantic relatedness methods, but of course this comparison is unfair because our method
utilizes labeled examples that must be paid for.

Our research leaves many questions and issues that we find interesting and worthy of further study.
Let us now mention a few.

\textbf{\emph{The making of a good BK corpus.}}
Our results indicate that high quality semantic relatedness can be learned with markedly different types of BK corpora.
In particular, we showed that semantic relatedness can be learned from a random and relatively small collection of ordinary fiction literature (ebooks in the Project Gutenberg).
However, we observe that the corpus ``quality'' affects both the starting performance and the learning rate. Specifically,
the starting performance, before even a single labeled example is introduced, is significantly higher when using Wikipedia as a BK corpus. In fact, this initial performance (obtained by $\nsd$ alone)
is by itself among the top performing \emph{unsupervised} methods.
Moreover, the learning rate obtained when using Wikipedia as a BK corpus, rather than
Project Gutenberg, is clearly faster.

An interesting question here is what makes a BK corpus useful for learning semantic relatedness?
Our speculative answer (yet to be investigated) is that a good corpus should consist of
semantically coherent contexts that span a wide scope of meanings.
For example, when generating the set of contexts from a fiction book, we can dissect the book into sentences, paragraphs, sections, etc.
Large contexts (say, sections) will include many more co-occurrence relations than small contexts (say, sentences), but among these
relations we expect to see entirely unrelated terms. On the other hand, if we only have very small contexts
we will to obtain only a subset of the related terms.
Thus, the context size directly affects the \emph{precision} and \emph{recall} of
observed ``meanings'' in a set of contexts.
The learning curves of Figure~\ref{Figure:Gutenberg-Spearman} hint on such a tradeoff when using the books in Project Gutenberg.

\textbf{\emph{Semantic relatedness between text fragments.}}
Many of the interesting applications mentioned in the introduction
can be solved given the ability to evaluate the relatedness between text fragments
(in this paper we only considered relatedness between individual terms).
One can think of many ways to extend any semantic relatedness measure from terms to text fragments,
and many such methods already proposed in the literature
\shortcite<see, e.g.,>{TsatsaronisVV10,Broder2007,Varelas2005}.
However, an interesting challenge would be to extract a semantic relatedness model using supervised learning where the training examples are relatedness preferences over text fragments. Such a model could be optimized to particular semantic tasks.
%To the best of our knowledge no such learning algorithm has been proposed.
%We believe that the task of developing such learning algorithm is the next step in the semantic relatedness learning research.

\textbf{\emph{Disambiguated semantic relatedness.}}
To the best of our knowledge, all the proposed term-based semantic relatedness methods discussed
in the literature follow a similar problem setup
where relatedness is evaluated regardless of particular context(s).
However, in many applications there exist such contexts that can and should be utilized.
For example, it is often the case where we have a target term along with its current context and we need to
rank the terms in our dictionary according to relatedness to this target term.
In such cases, the textual environment of the target term can be utilized to disambiguate it and
contribute to achieve better and more accurate \emph{contextual} relatedness evaluations.
It would be interesting to extend our model and methods to accommodate such contexts.

\textbf{\emph{Active learning.}}
In this work we proposed a passive learning algorithm that utilizes a uniformly sampled training set of preferences.
It would be very interesting to consider active learning techniques to cleverly sample training preferences and
expedite the learning process.
Assuming a realizable setting, and that preferences satisfy transitivity, a straightforward approach would be to use
a sorting algorithm to perfectly order $n$ term pairs
using $\Theta(n\log n)$ comparisons (training examples).
It is easy to argue that this is also an information theoretic lower bound on the sample complexity. Thus, several questions arise.
First, is it possible to approach this bound within an agnostic setting?
Second, is it possible to use some underlying structure (e.g., as exhibited in the BK corpus) to
achieve a sample complexity of $o ( n \log n)$? Finally, in many applications of interest we can do with ranking only
the top $k$
most similar terms to the target term. What would be the best theoretical and practically achievable sample complexities
in this case?
We note that a general active learning algorithm for preferences in the agnostic setting, guaranteeing $O(n \cdot polylog(n))$ sample
complexity, was very recently proposed by \citeA{Ailon2011} and \shortciteA{DBLP:journals/jmlr/AilonBE12}.

\textbf{\emph{Convergence and error bounds.}}
Regarding convergence of $\algname$, it is quite easy to see that the learning process of $\algname$ always converges.
This holds because
$\Delta_{prev}$ can only increase or stay the same for $\log_{2}(\alpha_{max})$ iterations.
This means that it is
effectively monotonic decreasing, and it is bounded below by $0$.
This $\alpha_{max}$ threshold was introduced to handle noisy (non-realizable) realistic scenarios.
The question is if the $\alpha_{max}$ is really necessary when the problem is realizable.
We conjecture that the answer to this question is ``yes'' because $\algname$ only updates contexts
in which both of the terms in question co-occur.
Error analysis is another direction that may shed light on the learning process and perhaps improve the algorithm.
It is interesting to address this question within both a statistical learning (see discussion
in Section~\ref{Section:ALearningTheoreticalPerspective}), and also under worst case considerations in the spirit of online learning.

\textbf{\emph{Benchmark datasets for semantic relatedness.}}
When considering problems involving preferences over thousands of terms, as perhaps required in large-scale commercial applications, some millions of
humanly annotated preferences are required.
In contrast, the academic semantic relatedness research is unfortunately solely relying on small sized annotated benchmark
datasets, such as WordSim353, which leaves much to be desired.
%The main problem in this limited size and scope is their inadequacy to generate results
Considering that the typical vocabulary of an English speaking
adult consists of several thousands words, a desired benchmark dataset should be of at least one or even two orders of magnitude larger than
WordSim353. While acquiring a sufficiently large semantic dataset can be quite costly, we believe that the semantic relatedness research will greatly benefit
once it will be introduced.
%This challenge can be clearly fulfilled by large commercial companies, but is
%perhaps beyond reach in academic environments.

While a formal understanding of meaning still seems to be beyond reach, we may be closer to a point where computer programs
are able to exhibit artificial understanding of meaning. Will large computational resources
to process huge corpora, together with a very large set of labeled training examples be sufficient?

\vskip 0.2in
\bibliography{supervised_learning_of_semantic_relatedness_extended}
\bibliographystyle{theapa}

\newpage
\appendix
\section*{Appendix A. The VC-dimension of an unrestricted hypothesis class}
\begin{lemma}
$VCdim\left( \cH_{all} \right) = \binom{ \binom{|D|}{2} }{2}$.
\end{lemma}
\begin{proof}
We say that $\{t_1,t_2\} \succ \{t_3,t_4\}$ if
$$
(t_1 > t_2) \wedge (t_3 > t_4) \wedge((t_1 > t_3) \vee ((t_1 = t_3) \wedge (t_2 > t_4))),
$$
where $>$ is the lexicographic order.
Thus, $\succ$ induces a complete order over term pairs, $\{t_1,t_2\}$ (with $ t_1 \neq t_2$),
because the lexicographic order induces a complete order.
Also, given a quadruple $x = (\{t_1,t_2\},\{t_3,t_4\})$, we define its \emph{inverse preference},
$$
\overline x \eqdef (\{t_3,t_4\},\{t_1,t_2\}).
$$
Recall the definition of $D_{pref}$ (\ref{eq:Dpref}), and let $S \subset D_{pref}$, be the set of quadruples satisfying
(it is not hard to see that $S$ is unique),
$$
S \eqdef
\{ \left(\{t_1,t_2\},\{t_3,t_4\}\right) \ \ | \ \ \{t_1,t_2\} \succ \{t_3,t_4\}\},
$$
As $\{t_1,t_2\} \neq \{t_3,t_4\}$, we have that either
$$
\{t_1,t_2\} \succ \{t_3,t_4\}
$$
or
$$
\{t_3,t_4\} \succ \{t_1,t_2\}.
$$
It follows that
$$
x \in S \Leftrightarrow \overline x \notin S.
$$

It is easy to see that
$|S| = \binom{ \binom{|D|}{2} }{2}$,
which holds because in each of the two term pairs the order between the different terms is fixed,
and the order between the different pairs themselves is fixed as well.

Consider any $X \subseteq S$, and set
$
\overline X \eqdef \{ x | \overline x \in X \}
$.
Now define
$$
h_X : D_{pref} \to \{\pm 1\},
$$ as follows
\begin{equation}
h_X(x) \eqdef \left\{
             \begin{array}{ll}
               -1, & \hbox{if $x \in X \cup \left(D_{pref} \setminus \left( S \cup \overline X \right) \right)$;} \\
               +1, & \hbox{otherwise.}
             \end{array}
           \right.
\end{equation}
To show that $h_X$ is a proper hypothesis, satisfying condition~(\ref{ref:cond}), we consider the following
mutually exclusive cases.\\
{\bf Case A:} $x \in X$. In this case we have,
\begin{eqnarray*}
x \in X & \Rightarrow & x \in X \cup \left(D_{pref} \setminus \left( S \cup \overline X \right) \right) \\
& \Rightarrow &  h_X(x) = -1 \\
x \in X & \Rightarrow & \overline x \in \overline X \\
& \Rightarrow & \overline x \notin X \cup \left(D_{pref} \setminus \left( S \cup \overline X \right) \right) \\
& \Rightarrow &  h_X(\overline x) = +1 \\
& \Rightarrow &  h_X(x) = -h_X(\overline x).
\end{eqnarray*}
{\bf Case B:} $x \in S \setminus X$. We now have,
\begin{eqnarray*}
x \in S \setminus X & \Rightarrow & x \notin X \cup \left(D_{pref} \setminus \left( S \cup \overline X \right) \right)\\
& \Rightarrow &  h_X(x) = +1 \\
x \in S \setminus X & \Rightarrow & (\overline x \notin S) \wedge (\overline x \notin \overline X) \\
& \Rightarrow &  \overline x \in  D_{pref} \setminus \left( S \cup \overline X \right) \\
& \Rightarrow &  h_X(\overline x) = -1 \\
& \Rightarrow &  h_X(x) = -h_X(\overline x).
\end{eqnarray*}
{\bf Case C:}
Now $x \notin S$ and $x \notin \overline X$,
\begin{eqnarray*}
(x \notin S) \wedge (x \notin \overline X) & \Rightarrow & x \in D_{pref} \setminus \left( S \cup \overline X \right) \\
& \Rightarrow &  h_X(x) = -1 \\
(x \notin S) \wedge (x \notin \overline X) & \Rightarrow & \overline x \in S \setminus X \\
& \Rightarrow &  h_X(\overline x) = +1 \\
& \Rightarrow &  h_X(x) = -h_X(\overline x).
\end{eqnarray*}
{\bf Case D:} In this case $x \notin S$ and $x \in \overline X$,
\begin{eqnarray*}
x \in \overline X & \Rightarrow & h_X(x) = +1 \\
x \in \overline X & \Rightarrow &  \overline x \in X \\
& \Rightarrow &  h_X(\overline x) = -1 \\
& \Rightarrow &  h_X(x) = -h_X(\overline x).
\end{eqnarray*}
To summarize, in all cases, and therefore for all $x \in D_{pref}$,
$$
h_X(x)=-h_X(\overline x),
$$
namely, $h_X$ satisfies condition~(\ref{ref:cond}), so $h_X \in \cH_{all}$.

We have that $\forall x \in X, h_X(x) = -1$,
and $\forall x \in S \setminus X, h_X(x) = +1$. Since $X$ is an arbitrary subset, $S$ is shattered by $\cH_{all}$.
Therefore,
$$
VCdim\left( \cH_{all} \right) \geq |S| = \binom{ \binom{|D|}{2} }{2}.
$$

Now assume, by contradiction, that there exists $S'$, satisfying
$|S'| >  \binom{ \binom{|D|}{2} }{2}$, and $S'$ is also shattered by $\cH_{all}$.
If this holds then there exists a classifier $h' \in \cH_{all}$, such
that $\forall x \in S', h'(x) = -1$.
Using the pigeonhole principle we get that there must be a quadruple
$x \in S'$ such that also $\overline x \in S'$ (and $x \neq \overline x$).
Thus,
$h'(x) = -1 = h(\overline x)$, contradicting the assumption that $h' \in \cH_{all}$.
Therefore, such a set $S'$ cannot exist, and $VCdim\left( \cH_{all} \right) \leq |S| = \binom{ \binom{|D|}{2} }{2}$,
which completes the proof.
\end{proof}

\section*{Appendix B. The VC-dimension of permutations}
\begin{lemma}
[\cite{RadinskyA11}]
$VCdim\left( \cH_{\pi} \right) = \binom{|D|}{2} - 1$.
\end{lemma}
\begin{proof}
Given any subset $S \subseteq D_{pref}$, we view $S$ as an undirected graph, $G = (V,E)$, whose node set $V$ consists of all
word pairs in the quadruples of $S$, and every edge in $E$ is associated with a quadruple in $S$
that connects its two pairs.
Any classifier, $h \in \cH$, defines directions for the edges of this graph in accordance with its preferences.

The maximum number of nodes in $G$ is $n = \binom{|D|}{2}$.
So, if $|S| = |E| > n-1$, the graph must contain an undirected cycle.
If $S$ is shattered by $\cH_{\pi}$, there must be a particular classifier $h \in \cH_{\pi}$
that classifies the same all the quadruples
in $S$. However, such a classifier creates a directed cycle in $G$.
Since $\cH_{\pi}$ contains only permutations,
its classifiers cannot induce directed cycles in $G$, and therefore, $S$ cannot be shattered by $\cH_{\pi}$.
It follows that $VCdim\left( \cH_{\pi} \right) \leq n-1 = \binom{|D|}{2} - 1$.

In case $|S| \leq \binom{|D|}{2} - 1$ and its underlying undirected graph $G$ is a forest,\footnote{From graph theory we know that such forests exists due to the edges number bound.}
then the edges can be directed in any desirable way, without creating a directed cycle.
In other words, $S$ can be shattered by $\cH_{\pi}$, and $VCdim\left( \cH_{\pi} \right) \geq \binom{|D|}{2} - 1$.
%Given that $\binom{|D|}{2} - 1 \leq VCdim\left( \cH_{\pi} \right) \leq \binom{|D|}{2} - 1$, we conclude that  $VCdim\left( \cH_{\pi}
%\right) = \binom{|D|}{2} - 1$.
\end{proof}

\end{document}